\documentclass{article}

\usepackage{microtype}
\usepackage{graphicx}
\usepackage{booktabs} 

\usepackage{amssymb}
\usepackage{amsmath}
\usepackage{amsthm}
\usepackage{subcaption}

\usepackage{hyperref}

\newtheorem{prop}{Proposition}

\newcommand{\RNum}[1]{\uppercase\expandafter{\romannumeral #1\relax}}

\usepackage[accepted]{icml2018}

\icmltitlerunning{Reviving and Improving Recurrent Back-Propagation}

\begin{document}

\twocolumn[
\icmltitle{Reviving and Improving Recurrent Back-Propagation}



\icmlsetsymbol{equal}{*}

\begin{icmlauthorlist}
\icmlauthor{Renjie Liao}{equal,A,B,C}
\icmlauthor{Yuwen Xiong}{equal,A,B}
\icmlauthor{Ethan Fetaya}{A,C}
\icmlauthor{Lisa Zhang}{A,C}
\icmlauthor{KiJung Yoon}{D,E}
\icmlauthor{Xaq Pitkow}{D,E}
\icmlauthor{Raquel Urtasun}{A,B,C}
\icmlauthor{Richard Zemel}{A,C,F}
\end{icmlauthorlist}

\icmlaffiliation{A}{Department of Computer Science, University of Toronto}
\icmlaffiliation{B}{Uber ATG Toronto}
\icmlaffiliation{C}{Vector Institute}
\icmlaffiliation{D}{Department of Electrical and Computer Engineering, Rice University}
\icmlaffiliation{E}{Department of Neuroscience, Baylor College of Medicine}
\icmlaffiliation{F}{Canadian Institute for Advanced Research}
\icmlcorrespondingauthor{Renjie Liao}{rjliao@cs.toronto.edu}

\icmlkeywords{Recurrent Back-Propagation, Deep learning, RNNs, Optimization}

\vskip 0.3in
]



\printAffiliationsAndNotice{\icmlEqualContribution} 

\begin{abstract}
In this paper, we revisit the recurrent back-propagation (RBP) algorithm~\cite{almeida1987learning,pineda1987generalization}, discuss the conditions under which it applies as well as how to satisfy them in deep neural networks. 
We show that RBP can be unstable and propose two variants based on conjugate gradient on the normal equations (CG-RBP) and Neumann series (Neumann-RBP).
We further investigate the relationship between Neumann-RBP and back propagation through time (BPTT) and its truncated version (TBPTT).
Our Neumann-RBP has the same time complexity as TBPTT but only requires constant memory, whereas TBPTT's memory cost scales linearly with the number of truncation steps.
We examine all RBP variants along with BPTT and TBPTT in three different application domains: associative memory with continuous Hopfield networks, document classification in citation networks using graph neural networks and hyperparameter optimization for fully connected networks.
All experiments demonstrate that RBPs, especially the Neumann-RBP variant, are efficient and effective for optimizing convergent recurrent neural networks.
Code is released at: \url{https://github.com/lrjconan/RBP}.
\end{abstract}

\section{Introduction}\label{sect:intro}

Back-propagation through time (BPTT) ~\cite{werbos1990backpropagation} is nowadays the standard approach for training recurrent neural networks (RNNs).
However, the computation and memory cost of BPTT scale linearly with the number of steps which makes BPTT impractical for applications where long sequences are common ~\cite{sutskever2014sequence,goodfellow2016deep}.
Moreover, as the number of unrolling steps increases, the numerical error accumulates which may render the algorithm useless in some applications, e.g., gradient-based hyperparameter optimization~\cite{maclaurin2015gradient}. 
This issue is often solved in practice by using truncated back-propagation through time (TBPTT)~\cite{williams1990efficient,sutskever2013training} which has constant computation and memory cost, is simple to implement, and effective in some applications. 
However, the quality of the TBPTT approximate gradient is not well understood.
A natural question to ask is, can we get better gradient approximations while still using the same computational cost as TBPTT?

Here will show that under certain conditions on the underlying model, the answer is positive.
In particular, we consider a class of RNNs whose hidden state converges to a steady state.
For this class of RNNs, we can bypass BPTT and compute the exact gradient using an algorithm called recurrent back-propagation (RBP)~\cite{almeida1987learning,pineda1987generalization}.
The key observation exploited by RBP is that the gradient of the steady state w.r.t. the learnable parameters can be directly computed using the implicit function theorem, alleviating the need to unroll the entire forward pass. 
The main computational cost of RBP is in solving a linear system which has constant memory and computation time w.r.t. the number of unrolling steps. However, due to the strong assumptions that RBP imposes, TBPTT has become the standard approach used in practice and RBP did not get much attention for many years.

In this paper, we first revisit RBP in the context of modern deep learning.
We discuss the original algorithm, the assumptions it imposes and how to satisfy them for deep neural networks.
Second, we notice that although the fixed point iteration method used in~\cite{almeida1987learning,pineda1987generalization} is guaranteed to converge if the steady hidden state is achievable, in practice it can fail to do so within a reasonable amount of steps. 
This may be caused by the fact that there are many fixed points and the algorithm is sensitive to initialization.
We try to overcome the instability issue by proposing two variants of RBP based on conjugate gradient on normal equations (CG-RBP) and Neumann series (Neumann-RBP). 
We show a connection between Neumann-RBP and TBPTT which sheds some new light on the approximation quality of TBPTT.
In the experiments, we show several important applications which are naturally amenable to RBP. 
For example, we show how RBP can be used to back propagate thorough the optimization of deep neural networks in order to tune hyperparameters.
Throughout our experiments, we found that Neumann-RBP not only inherits the advantages of original RBP but also remains consistently stable across different applications.

\section{Related Work}\label{sect:related}


In the context of neural networks, RBP was independently discovered by Almeida~\cite{almeida1987learning} and Pineda~\cite{pineda1987generalization} in 1987, which is why this algorithm is sometimes called the Almeida-Pineda algorithm. 
Back then, RBP was shown to be useful in learning content-addressable memory (CAM) models~\cite{hopfield1982neural,hopfield1984neurons} and other computational neurodynamic models~\cite{lapedes1986self,haykinneural,chauvin1995backpropagation}.
These models are special RNNs in a sense that their inference stage is a convergent dynamic system by design. 
For these systems, one can construct a Lyapunov function for the underlying dynamics which further guarantees the asymptotic stability.
We refer readers to chapter $13$ of~\cite{haykinneural} for more details on neurodynamic models. 
The goal of learning in these models is to manipulate the attractors, {\it i.e.} steady states, such that they are close to the input data.
Therefore, during inference stage, even if the input data is corrupted, the corresponding correct attractor or ``memory`` can still be retrieved.
Instead of computing gradient via BPTT, RBP provides an more efficient alternative for manipulating the attractors.

RBP was later applied to learning graph neural networks (GNNs)~\cite{scarselli2009graph}, which are generalizations of RNNs that handle graph-structured input data. 
Specifically, the inference of GNNs is essentially a propagation process which spreads information along the graph.
One can force the propagation process to converge by either constructing a contraction map explicitly, or by regularizing the Jacobian of the update function.
Similarly, the goal is to push the converged inference solution close to the target.
RBP is naturally applicable here and demonstrated to save both computation time and memory.
A recent investigation~\cite{benjamin2017equivalence} shows that RBP is related to equilibrium propagation~\cite{scellier2017equilibrium} which is motivated from the perspective of biological plausibility.
Another recent related work in deep learning is OptNet~\cite{amos2017optnet} where the gradient of the optimized solution of a quadratic programming problem w.r.t. parameters is obtained by analytically differentiating the KKT system.

In the probabilistic graphical models (PGMs) literature, similar techniques to RBP have been developed as well.
For example, an efficient gradient-based method to learn the hyperparameters of log-linear models is provided in~\cite{foo2008efficient} where the core contribution is to use the implicit differentiation trick to compute the gradient of the optimized inference solution w.r.t. the hyperparameters.
A similar implicit differentiation technique is used in~\cite{samuel2009learning} to optimize the maximum a posterior (MAP) solution of continuous MRFs, since the MAP solution can be regarded as the steady state of the inference process.
An implicit-differentiation-based optimization method for generic energy models is proposed in~\cite{domke2012generic} where the gradient of the optimal state (steady state) of the energy w.r.t. the parameters can be efficiently computed given the fast matrix-vector product implementation~\cite{pearlmutter1994fast}.
If one regards the inference algorithms from aforementioned applications as unrolled RNNs, the implicit differentiation technique is essentially equivalent to RBP.

Other efforts have been made to develop alternatives to BPTT.
NoBackTrack~\cite{ollivier2015training} maintains an online estimate of the gradient via the random rank-one reduction technique. 
ARTBP~\cite{tallec2017unbiasing} introduces a probability distribution over the truncation points in the sequence and compensates the truncated gradient based on the distribution. 
Both approaches provide an unbiased estimation of the gradient although their variances differ.

\section{Revisiting Recurrent Back-Propagation}\label{sect:model_bg}

In this section, we review the original RBP algorithm and discuss its assumptions.

\subsection{Recurrent Back-Propagation}

We denote the input data and initial hidden state as $x$ and $h_0$.
During inference, the hidden state at time $t$ is computed as follows,
\begin{align}\label{eq:update}
h_{t+1} = F(x, w_F, h_{t}),
\end{align}
where $F$ is the update function parameterized by $w_F$.
A typical instantiation of $F$ is an LSTM~\cite{hochreiter1997long} cell function.
This RNN formulation differs from the one commonly used in language modeling, as the input is not time-dependent.
We restrict our attention to RNNs with fixed inputs for now as it requires fewer assumptions.
Assuming the dynamical system, (i.e., the forward pass of the RNN), reaches steady state before time step $T$, we have the following equation, 
\begin{align}\label{eq:fixed_point}
h^{\ast} = F(x, w_F, h^{\ast}),
\end{align}
where $h^{\ast}$ is the steady hidden state.
We compute the predicted output $y$ based on the steady hidden state as follows,
\begin{align}\label{eq:output}
y = G(x, w_G, h^{\ast}),
\end{align}
where $G$ is the output function parameterized by $w_G$.
Typically, a loss function $L = l(\bar{y}, y)$ measures the closeness between ground truth $\bar{y}$ and predicted output $y$.
Since the input data $x$ is fixed for all time steps, we can construct a function $\Psi$ of $w_F$ and $h$ as follows,
\begin{align}\label{eq:psi_func}
\Psi(w_F, h) = h - F(x, w_F, h).
\end{align}

At the fixed point, we have $\Psi(w_F, h^{\ast}) = 0$.
Assuming some proper conditions on $F$, e.g., continuous differentiability,
we can take the derivative w.r.t. $w_F$ at $h^{\ast}$ on both sides.
Using the total derivative and the dependence of $h^{\ast}$ on $w_F$ we obtain,
\begin{align}\label{eq:psi_grad}
\frac{\partial \Psi(w_F, h^{\ast}) }{\partial w_F} & = \frac{\partial h^{\ast}}{\partial w_F} - \frac{\mathrm{d} F(x, w_F, h^{\ast}) }{\mathrm{d} w_F} \nonumber \\
& = \left(I - J_{F, h^{\ast}} \right) \frac{\partial h^{\ast}}{\partial w_F} - \frac{\partial F(x, w_F, h^{\ast}) }{\partial w_F} \nonumber \\
& = \mathbf{0}, 
\end{align}
where $J_{F, h^{\ast}} = \frac{\partial F(x, w_F, h^{\ast}) }{\partial h}$ is the Jacobian matrix of $F$ evaluated at $h^{\ast}$ and $\mathrm{d}$ is the total derivative operator.
Assuming that $I - J_{F, h^{\ast}}$ is invertible, we rearrange Eq.~(\ref{eq:psi_grad}) to get,
\begin{align}\label{eq:implicit_grad}
\frac{\partial h^{\ast}}{\partial w_F} = \left(I - J_{F, h^{\ast}} \right)^{-1} \frac{\partial F(x, w_F, h^{\ast})}{\partial w_F}.
\end{align}
In fact, Equations~(\ref{eq:psi_func}-~\ref{eq:implicit_grad}) are an application of the Implicit Function Theorem~\cite{rudin1964principles}, which guarantees the existence and uniqueness of an implicit function $\phi$ such that $h^{\ast} = \phi(w_F)$ if two conditions hold:
\RNum{1}, $\Psi$ is continuously differentiable and
\RNum{2}, $I - J_{F, h^{\ast}}$ is invertible.
Although we do not know the analytic expression of the function $\phi$, we can still compute its gradient at the fixed point.
Based on Eq.~(\ref{eq:implicit_grad}), we now turn our attention towards
computing the gradient of the loss w.r.t. the parameters of the RNN.
By using the total derivative and the chain rule, we have
\begin{align}
\frac{\partial L}{\partial w_G} & = \frac{\partial L}{\partial y} \frac{\partial G(x, w_G, h^{\ast})}{\partial w_G} \label{eq:rbp_grad_1} \\
\frac{\partial L}{\partial w_F} & = \frac{\partial L}{\partial y} \frac{\partial y}{\partial h^{\ast}} \left(I - J_{F, h^{\ast}} \right)^{-1} \frac{\partial F(x, w_F, h^{\ast})}{\partial w_F}.  \label{eq:rbp_grad_2}
\end{align}
Since the gradient of the loss w.r.t. $w_G$ can be easily obtained by back-propagation, we focus our exposition on the computation of $\frac{\partial L}{\partial w_F}$.
The original RBP algorithm~\cite{pineda1987generalization,almeida1987learning} introduces an auxiliary variable $z$ such that,
\begin{align}\label{eq:aux_var}
z = \left(I - J_{F, h^{\ast}}^{\top} \right)^{-1} \left( \frac{\partial L}{\partial y} \frac{\partial y}{\partial h^{\ast}} \right)^{\top},
\end{align}
where $z$ is a column vector.
If we managed to compute $z$, then we can substitute it into Eq. (\ref{eq:rbp_grad_2}) to get the gradient.
Note that the Jacobian matrix $J_{F, h^{\ast}}$ is nonsymmetric for general RNNs which renders direct solvers of linear system impractical.
To compute $z$, the original RBP algorithm uses fixed point iteration.
In particular, we multiply $\left(I - J_{F, h^{\ast}}^{\top} \right)$ on the left hand of both sides of Eq. (\ref{eq:aux_var}) and rearrange the terms as follows,
\begin{align}\label{eq:linear_sys_rbp}
z = J_{F, h^{\ast}}^{\top} z  + \left( \frac{\partial L}{\partial y} \frac{\partial y}{\partial h^{\ast}} \right)^{\top}.
\end{align}
If we view the right hand side of the above equation as a function of $z$, then
applying the fixed point iteration results in the Algorithm~\ref{alg:rbp}.
Note that the most expensive operation in this algorithm is the matrix-vector product $J_{F, h^{\ast}}^{\top}z$, which is the same operator as back-propagation.

\begin{algorithm}[t]
\caption{: Original RBP}
\label{alg:rbp}
\begin{algorithmic}[1]
\STATE \textbf{Initialization:} initial guess $z_{0}$, {\it e.g.}, draw uniformly from $[0,1]$, $i = 0$, threshold $\epsilon$
\REPEAT
\STATE $i = i + 1$
\STATE $z_{i} = J_{F, h^{\ast}}^{\top} z_{i-1}  + \left( \frac{\partial L}{\partial y} \frac{\partial y}{\partial h^{\ast}} \right)^{\top}$
\UNTIL{$\Vert z_{i} - z_{i-1} \Vert < \epsilon$}
\STATE $\frac{\partial L}{\partial w_F} = z_{i}^{\top} \frac{\partial F(x, w_F, h^{\ast})}{\partial w_F}$
\STATE Return $\frac{\partial L}{\partial w_F}$
\end{algorithmic}
\end{algorithm}

\subsection{Assumptions of RBP}

In this section, we discuss how to satisfy the assumptions of RBP.
Recall that in order to apply the implicit function theorem, $\Psi(w_F, h)$ has to satisfy two assumptions: 
\RNum{1}, $\Psi$ is continuously differentiable. 
\RNum{2}, $I - J_{F, h^{\ast}}$ is invertible.
Condition \RNum{1} requires the derivative of $F$ to be continuous, a condition satisfied by many RNNs, like LSTM and GRU~\cite{cho2014learning}.
Condition \RNum{2} is equivalent to requiring the determinant of $I - J_{F, h^{\ast}}$ to be nonzero,
i.e., $\text{det}(I - J_{F, h^{\ast}}) \neq 0$.
One sufficient but not necessary condition to ensure this is to force $F$ to be a contraction map, as in~\citet{scarselli2009graph}.
Recall that $F$ is a contraction map on Banach space $B$, i.e., a complete normed vector space, iff, $\forall h_1, h_2 \in B$, $\Vert F(h_1) - F(h_2) \Vert \leq \mu \Vert h_1 - h_2 \Vert$ where $0 \leq \mu < 1$.
Banach fixed point theorem guarantees the uniqueness of the fix point of the contraction map $F$ in $B$.
Note that here we drop the dependency of $F$ on $w$ for readability.
Based on the first order Taylor approximation, $F(h) = F(h^{\ast}) + J_{F, h^{\ast}}(h - h^{\ast})$, we have,
\begin{align}
\frac{\Vert F(h) - F(h^{\ast}) \Vert}{\Vert h - h^{\ast} \Vert} = \frac{\Vert J_{F, h^{\ast}}(h - h^{\ast}) \Vert}{\Vert h - h^{\ast} \Vert}.
\end{align}
Note that if we use $L_2$ vector norm, then the induced matrix norm, a.k.a., operator norm, is,
{\footnotesize
\begin{align}
\Vert J_{F, h^{\ast}} \Vert & = \text{sup} \left\{ \frac{\Vert J_{F, h^{\ast}} h \Vert}{\Vert h \Vert} : \forall h \neq 0 \right\} = \sigma_{\text{max}}(J_{F, h^{\ast}}),
\end{align}}
where $\sigma_{\text{max}}$ is the largest singular value.
Therefore, relying on the contraction map definition, we have,
\begin{align}\label{eq:contraction_jacobian}
\Vert J_{F, h^{\ast}} \Vert \leq \mu < 1,
\end{align}
Moreover, since the minimum singular value of $I - J_{F, h^{\ast}}$ is $1 - \sigma_{\text{max}}(J_{F, h^{\ast}})$, we have  
\begin{align}\label{eq:contraction_map_correctness}
\vert \text{det}(I - J_{F, h^{\ast}}) \vert & = \prod_{i} \vert \sigma_{i}(I - J_{F, h^{\ast}}) \vert \nonumber \\
& \geq \left[ 1 - \sigma_{\text{max}}(J_{F, h^{\ast}}) \right]^{d} > 0.
\end{align}
Thus our second condition holds following Eq.~(\ref{eq:contraction_map_correctness}).

\citet{scarselli2009graph} use $L_1$ vector norm which results in a looser inequality since $\Vert J_{F, h^{\ast}} \Vert_2 \leq \sqrt{d} \Vert J_{F, h^{\ast}} \Vert_1 $.
They obtain an easier to compute regularization term $\max_{i}( \Vert J_{F, h^{\ast}}(:, i) \Vert_1 - \eta)^{2}$ where $(:, i)$ denotes the $i$-th column and $\eta \in (0, 1)$ is the desired contraction constant.
We note, however, that this work makes a claim that the contraction map assumption can be achieved by regularizing the local Jacobian $J_{F, h^{\ast}}$ of a general neural network.
This is problematic because the contraction map property is a global property of $F$ that requires
regularizing every $h$ in the space $B$, not just $h^{\ast}$. %
Nevertheless, this regularization evaluated at $h^{\ast}$ encourages local contraction at the fixed point $h^{\ast}$, which is sufficient for satisfying condition \RNum{2}. 
Another way to enforce condition \RNum{2} to hold is directly formalizing the Lagrangian of equality constraint $\Psi(w_F, h^{\ast}) = 0$.
Since all applications we considered in this paper have converged dynamic systems in practice, we leave further discussions 
of condition II to the appendix.

\section{New Recurrent Back-Propagation Variants}\label{sect:model}

In this section, we present our newly proposed variants of RBP, CG-RBP and Neumann-RBP, in detail.

\subsection{Recurrent Back-Propagation based on Conjugate Gradient}

Facing the system of linear equations like Eq.~(\ref{eq:aux_var}) in the derivation of original RBP, one would naturally think of the most common iterative solver, i.e., conjugate gradient method~\cite{hestenes1952methods}.
In particular, multiplying $I - J_{F, h^{\ast}}^{\top}$ on both sides, we obtain the following equations,
\begin{align}\label{eq:linear_sys}
\left(I - J_{F, h^{\ast}}^{\top} \right) z = \left(\frac{\partial L}{\partial y} \frac{\partial y}{\partial h^{\ast}}\right)^{\top}.
\end{align}
Unfortunately, for general RNNs, the Jacobian matrix $J_{F, h^{\ast}}$ of the update function, e.g., a cell function of LSTM, is non-symmetric in general.
This increases the difficulty of solving the system.
One simple yet sometimes effective way to approach this problem is to exploit the conjugate gradient method on the normal equations (CGNE)~\cite{golub2012matrix}.
Specifically, we multiply $I - J_{F, h^{\ast}}$ on both sides of Eq.~(\ref{eq:linear_sys}) which results in,
\begin{align}\label{eq:normal_linear_sys}
\left(I - J_{F, h^{\ast}} \right) \left(I - J_{F, h^{\ast}}^{\top} \right) z = \left(I - J_{F, h^{\ast}} \right) \left(\frac{\partial L}{\partial y} \frac{\partial y}{\partial h^{\ast}}\right)^{\top}. \nonumber
\end{align}
Having a symmetric matrix multiplying $z$ on the left hand side, we can now use the conjugate gradient method.
The detailed algorithm is easily obtained by instantiating the standard conjugate gradient (CG) template.
The most expensive operation used in CGNE is $J_{F, h^{\ast}}J_{F, h^{\ast}}^{\top}z$, which
can be implemented by successive matrix-vector products similarly for computing the Fisher
information product of the natural gradient method~\cite{schraudolph2002fast}.
Once we solve $z$ via $K$-step CGNE, we obtain the final gradient by substituting the solution into Eq.~(\ref{eq:rbp_grad_2}).
Since the condition number of the current system is the square of the original one,
the system may be slower to converge in practice.
Exploring more advanced and faster convergent numerical methods under this setting, like LSQR~\cite{paige1982lsqr}, is left for future work.

\subsection{Recurrent Back-Propagation based on Neumann Series}

We now develop a new RBP variant called Neumann-RBP,
which uses Neumann series from functional analysis and is efficient in terms of computation and memory. 
We then show its connections to BPTT and TBPTT.

A Neumann series is a mathematical series of the form $\sum_{t=0}^{\infty} A^{t}$ where $A$ is an operator.
In matrix theory, it is also known as the geometric series of a matrix. 
A convergent Neumann series has the following property,
\begin{align}
(I - A)^{-1} = \sum_{k=0}^{\infty} A^{k}.
\end{align}
One sufficient condition of convergence is that the spectral radius (i.e., the largest absolute eigenvalue value) of $A$ is less than $1$.
This convergence criterion applied to $A = J_{F, h^{\ast}}$ implies condition \RNum{2}.
Other cases where the convergence hold is beyond the scope of this paper.
If the Neumann series $\sum_{t=0}^{\infty} J_{F, h^{\ast}}^{t}$ converges, we can use it to replace
the term $\left(I - J_{F, h^{\ast}} \right)^{-1}$ in E.q.~(\ref{eq:rbp_grad_2}).
Furthermore,
the gradient of RBP can be approximated with the $K$-th order truncation of Neumann series as below,
\begin{align}
\frac{\partial L}{\partial w_F} \approx \frac{\partial L}{\partial y} \sum_{k=0}^{K} \frac{\partial y}{\partial h^{\ast}} J_{F, h^{\ast}}^{k} \frac{\partial F(x, w_F, h^{\ast})}{\partial w_F}.
\end{align}
There is a rich body of literature on how to compute Neumann series efficiently 
using binary or ternary decomposition~\cite{westreich1989eval,dimitrov2017on}. 
However, these decomposition based approaches are inapplicable in our context since we cannot compute the Jacobian matrix $J_{F, h^{\ast}}$
efficiently for general neural networks.
Fortunately, we can instead efficiently compute the matrix-vector product $J_{F, h^{\ast}}^{\top}u$ and $J_{F, h^{\ast}}u$ ($u$ is a proper sized vector) by using reverse and forward mode auto-differentiation~\cite{pearlmutter1994fast}.
Relying on this technique, we summarize the Neumann series based RBP algorithm in Algorithm~\ref{alg:neumann_rbp}.
In practice, we can obtain further memory efficiency by performing updates within the for loops in-place (please refer to the example code in appendix),
so that memory usage need not scale with the number of truncation steps.
Moreover, since the algorithm does not rely on hidden states except the steady state $h^{\ast}$, we no longer need to store the hidden states in the forward pass of the RNN. 
Besides the computational benefit, we now have the following propositions to connect Neumann-RBP to BPTT and TBPTT.
\begin{prop}\label{prop:neumann_rbp_1}
Assume that we have a convergent RNN which satisfies the implicit function theorem conditions.
If the Neumann series $\sum_{t=0}^{\infty} J_{F, h^{\ast}}^{t}$ converges, then the full Neumann-RBP is equivalent to BPTT. 
\end{prop}
\begin{prop}\label{prop:neumann_rbp_2}
For the above RNN, let us denote its convergent sequence of hidden states as $h^{0}, h^{1}, \dots, h^{T}$ where $h^{\ast} = h^{T}$ is the steady state.
If we further assume that there exists some step $K$ where $0 < K \le T$ such that $h^{\ast} = h^{T} = h^{T-1} = \dots = h^{T-K}$, then $K$-step Neumann-RBP is equivalent to $K$-step TBPTT.
\end{prop}

Moreover, the following proposition bounds the error of $K$-step Neumann-RBP.
\begin{prop}
If the Neumann series $\sum_{t=0}^{\infty} J_{F, h^{\ast}}^{t}$ converges, then the error between $K$-step and full Neumann series is as follows,
{\footnotesize
\begin{align}
\left\Vert \sum_{t=0}^{K} J_{F, h^{\ast}}^{t} - (I - J_{F, h^{\ast}})^{-1} \right\Vert \le \left\Vert (I - J_{F, h^{\ast}})^{-1} \right\Vert \left\Vert J_{F, h^{\ast}} \right\Vert^{K+1} \nonumber
\end{align}}
\end{prop}
We leave all proofs in appendix.







\begin{algorithm}[t]
\caption{: Neumann-RBP}
\label{alg:neumann_rbp}
\begin{algorithmic}[1]
\STATE \textbf{Initialization:} $v_{0} = g_{0} = \left( {\frac{\partial L}{\partial y}} {\frac{\partial y}{\partial h^{\ast}}} \right)^{\top}$
\FOR{step $t = 1, 2, \dots, K$}
\STATE $v_{t} = J^{\top} v_{t-1}$
\STATE $g_{t} = g_{t-1} + v_{t}$
\ENDFOR
\STATE $\frac{\partial L}{\partial w_F} = \left( g_{K} \right)^{\top} \frac{\partial F(x, w_F, h^{\ast})}{\partial w_F}$
\STATE Return $\frac{\partial L}{\partial w_F}$
\end{algorithmic}
\end{algorithm}

\section{Experiments}\label{sect:exp}

In this section, we thoroughly study all RBP variants on diverse applications.
Our implementation based on PyTorch is publicly available\footnote{\url{https://github.com/lrjconan/RBP}}.
Note that our Neumann-RBP is very simple to implement using automatic differentiation and we provide a very short example program in the appendix.

\subsection{Associative Memory}


\begin{figure}[t]
\renewcommand*{\arraystretch}{0.1}
\centering
\begin{tabular}{c@{\hspace{0.1mm}}c}
    \includegraphics[width=0.49\linewidth]{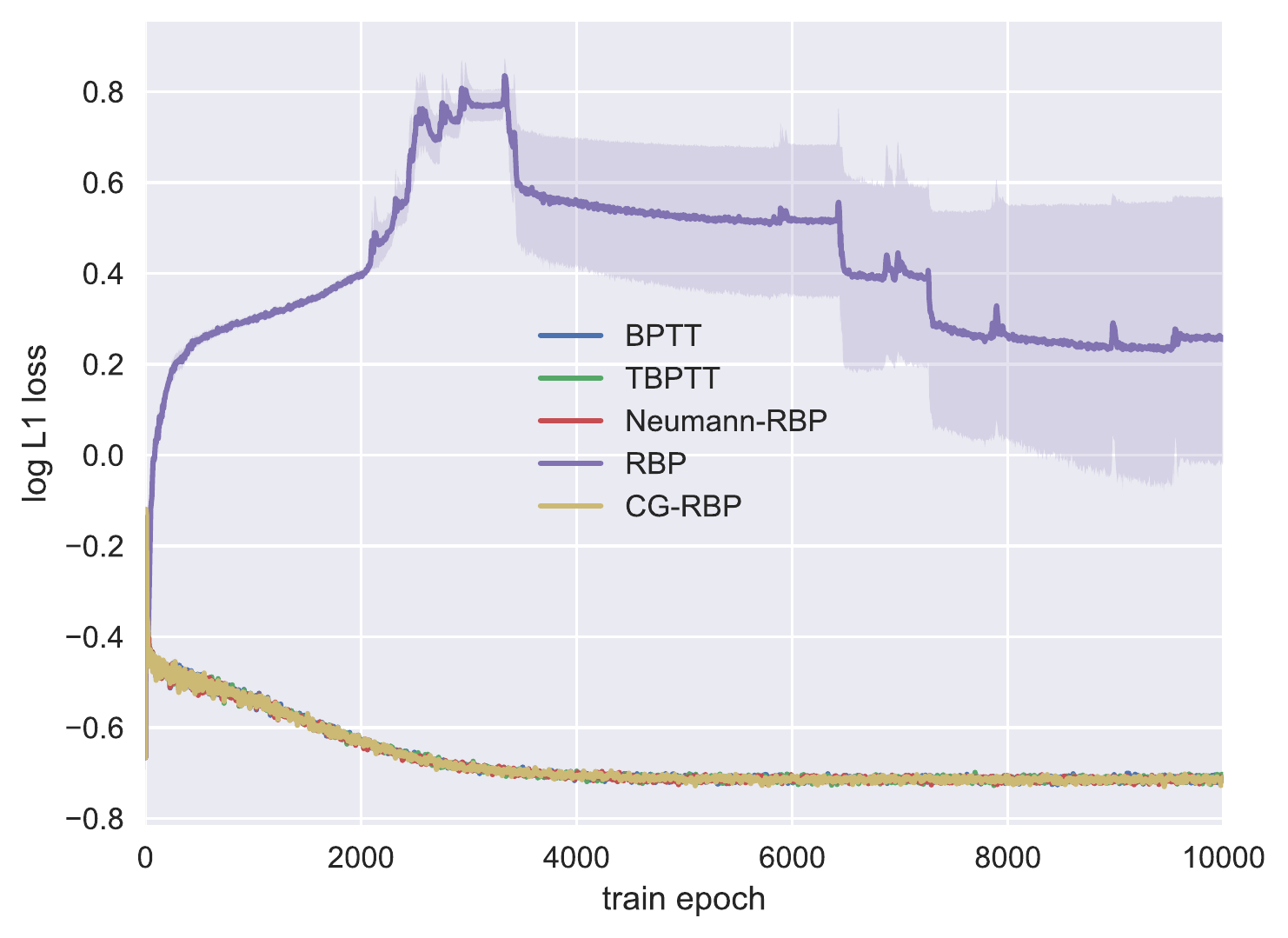}&
    \includegraphics[width=0.49\linewidth]{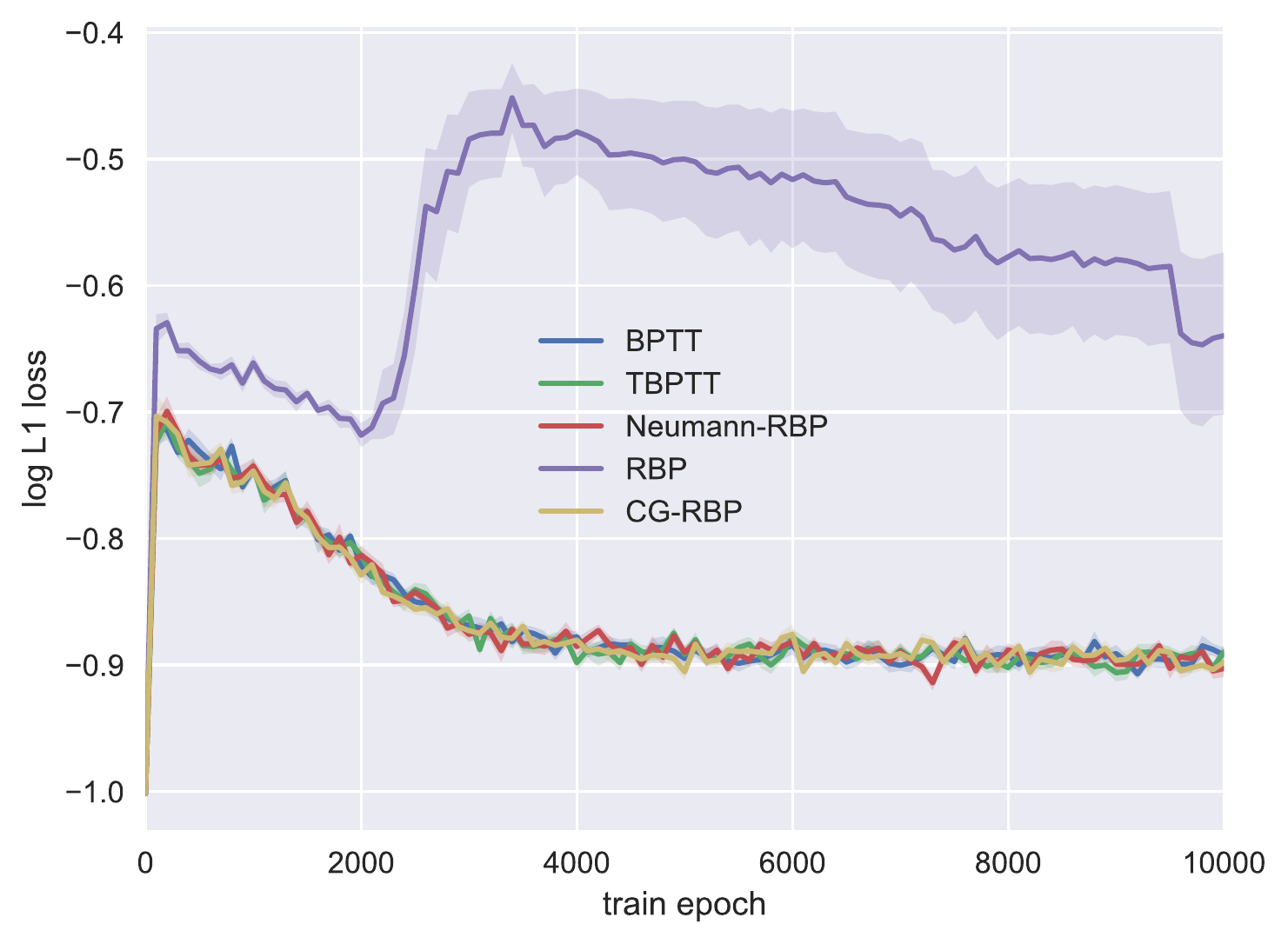} \\
\end{tabular}
\vspace{-0.3cm}
\caption{Left and right figures are training and validation curves of the same Hopfield network. $y$ axis is the log scale L1 loss. $x$ axis of (a) and (b) are training and validation step respectively. We do validation every $10$ training steps.} 
\label{fig:hopfield_curve}
\end{figure}

\begin{figure}
\centering
\renewcommand*{\arraystretch}{0.1}
\begin{tabular}{@{\hspace{0.1mm}}c@{\hspace{0.1mm}}c@{\hspace{0.1mm}}c@{\hspace{0.1mm}}c@{\hspace{0.1mm}}c@{\hspace{0.1mm}}c}
    \includegraphics[width=0.16\linewidth]{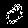}&
    \includegraphics[width=0.16\linewidth]{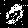}&
    \includegraphics[width=0.16\linewidth]{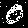}&
    \includegraphics[width=0.16\linewidth]{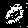}&
    \includegraphics[width=0.16\linewidth]{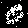} &
    \includegraphics[width=0.16\linewidth]{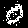}\\
    \includegraphics[width=0.16\linewidth]{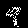}&
    \includegraphics[width=0.16\linewidth]{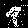}&
    \includegraphics[width=0.16\linewidth]{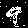}&
    \includegraphics[width=0.16\linewidth]{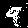}&
    \includegraphics[width=0.16\linewidth]{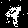} &
    \includegraphics[width=0.16\linewidth]{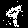}\\ [1mm]
    (a) & (b) & (c) & (d) & (e) & (f) \\
\end{tabular}
\vspace{-0.3cm}
\caption{Visualization of associative memory. (a) Corrupted input image; (b)-(f) are retrieved images by BPTT, TBPTT, RBP, CG-RBP, Neumann-RBP respectively.} 
\vspace{-0.5cm}
\label{fig:hopfield_vis}
\end{figure}


\begin{figure*}[t]
\renewcommand*{\arraystretch}{0.1}
\centering
\begin{tabular}{@{\hspace{0.1mm}}c@{\hspace{0.1mm}}c@{\hspace{0.1mm}}c}
    \includegraphics[width=0.32\linewidth]{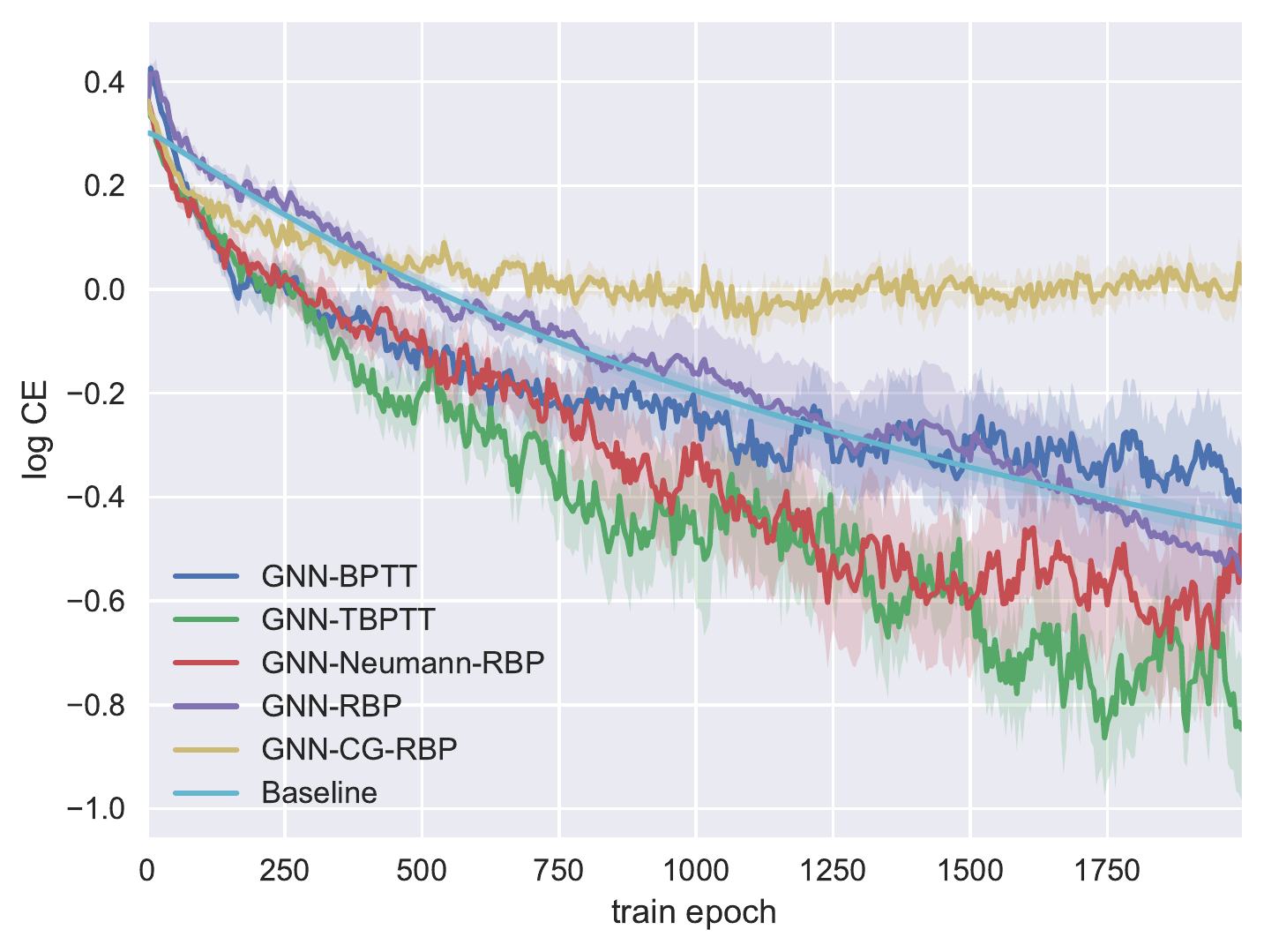}&
    \includegraphics[width=0.32\linewidth]{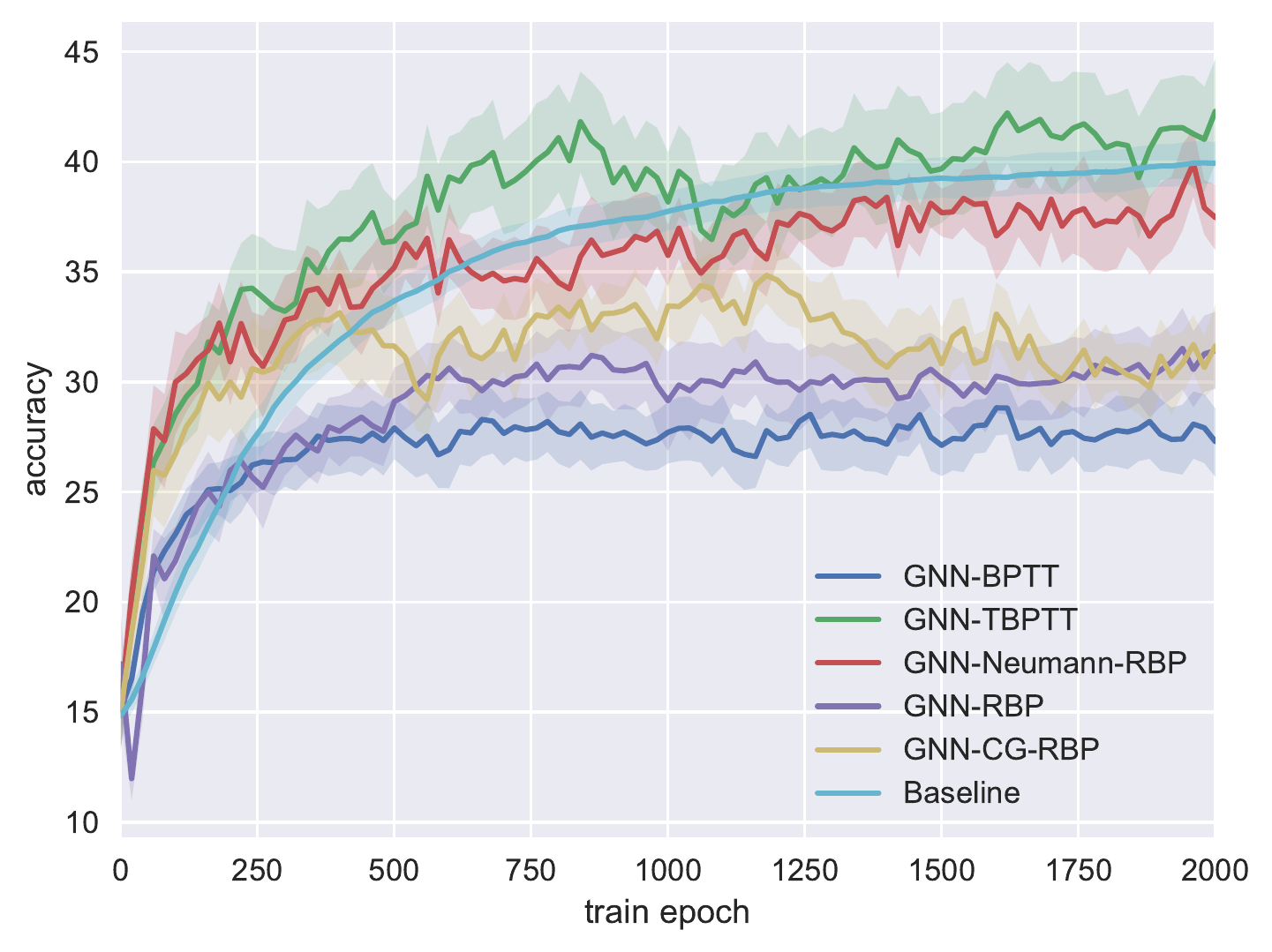}&
    \includegraphics[width=0.31\linewidth]{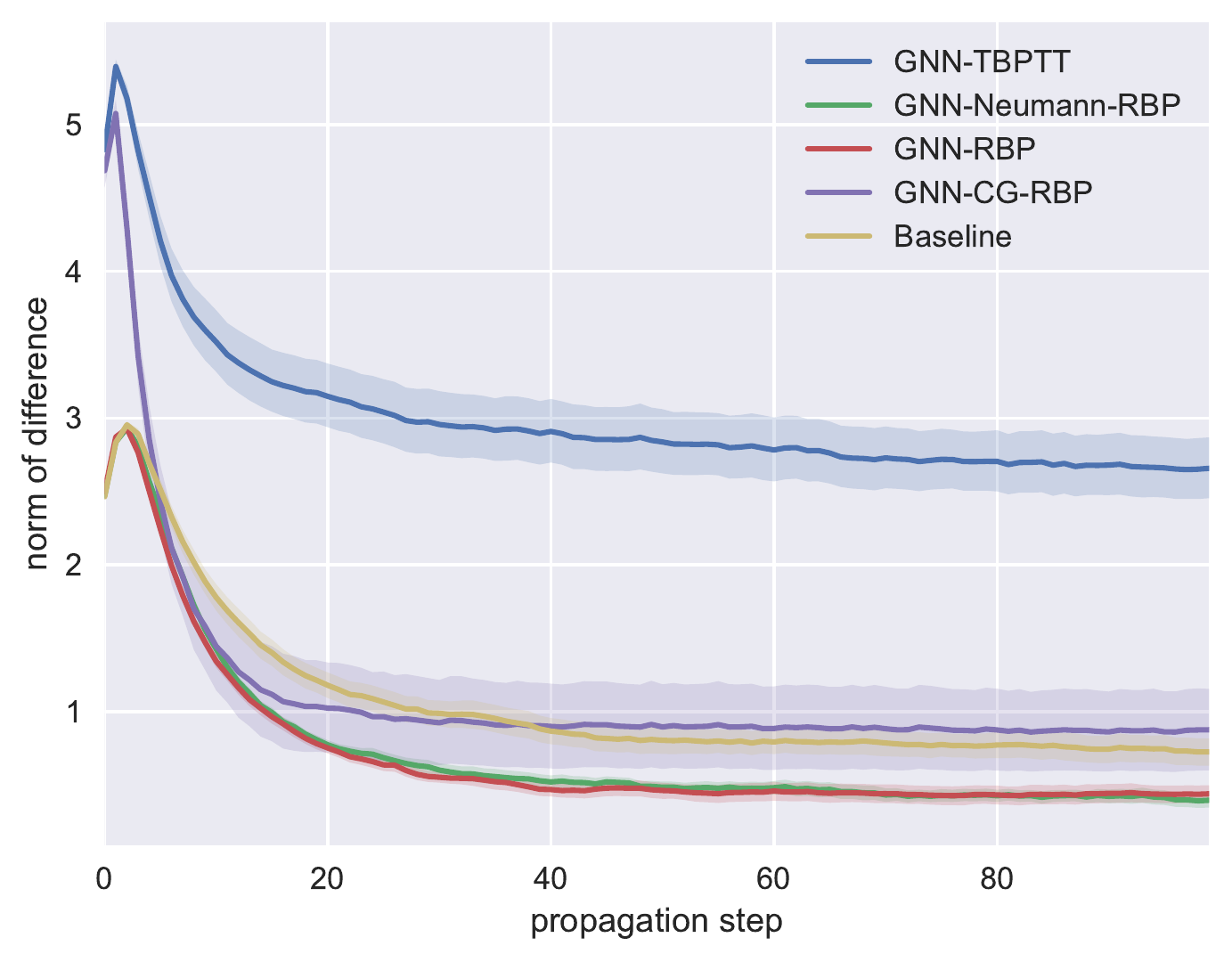}\\
    (a) & (b) & (c) \\
    \includegraphics[width=0.32\linewidth]{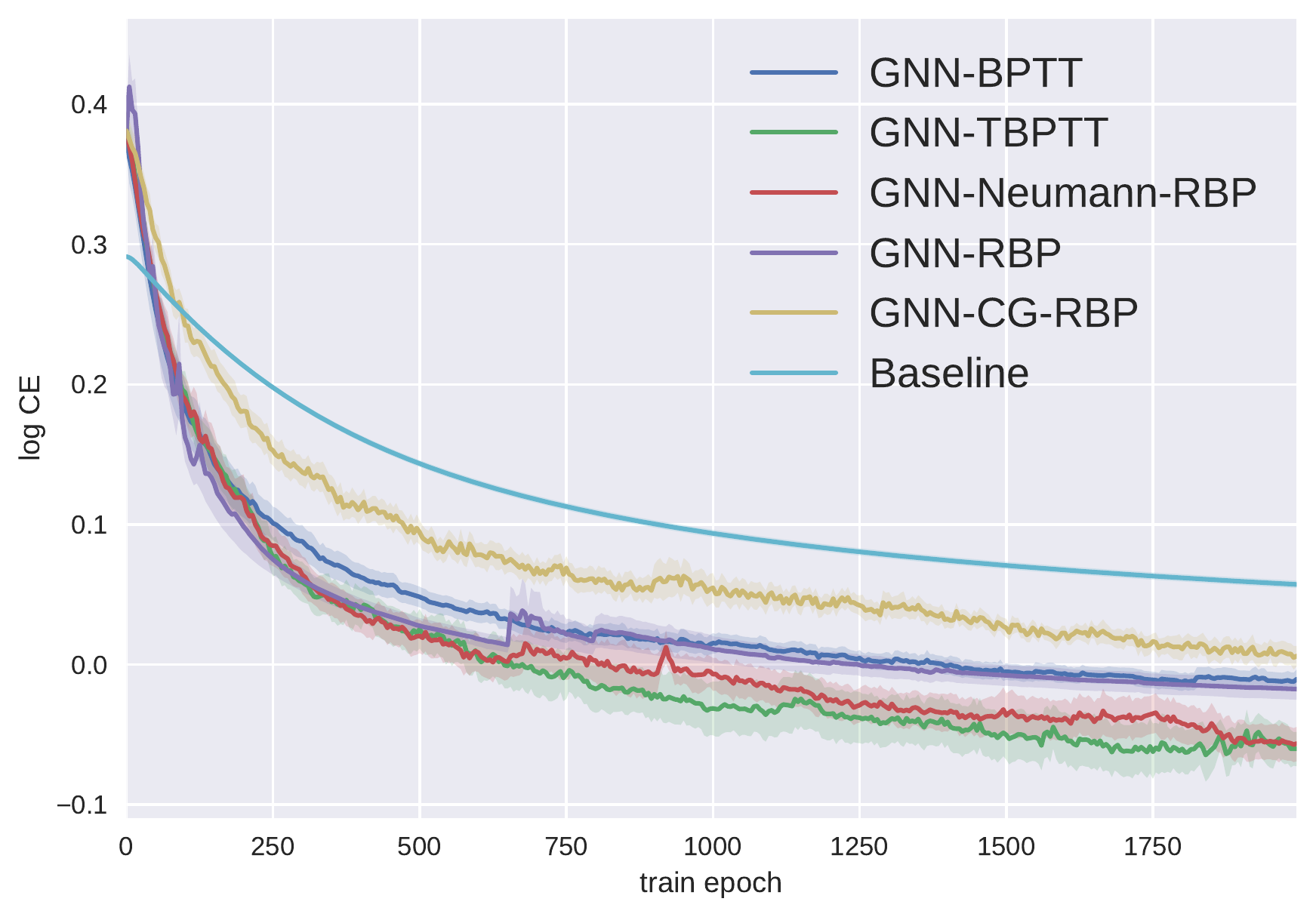}&
    \includegraphics[width=0.315\linewidth]{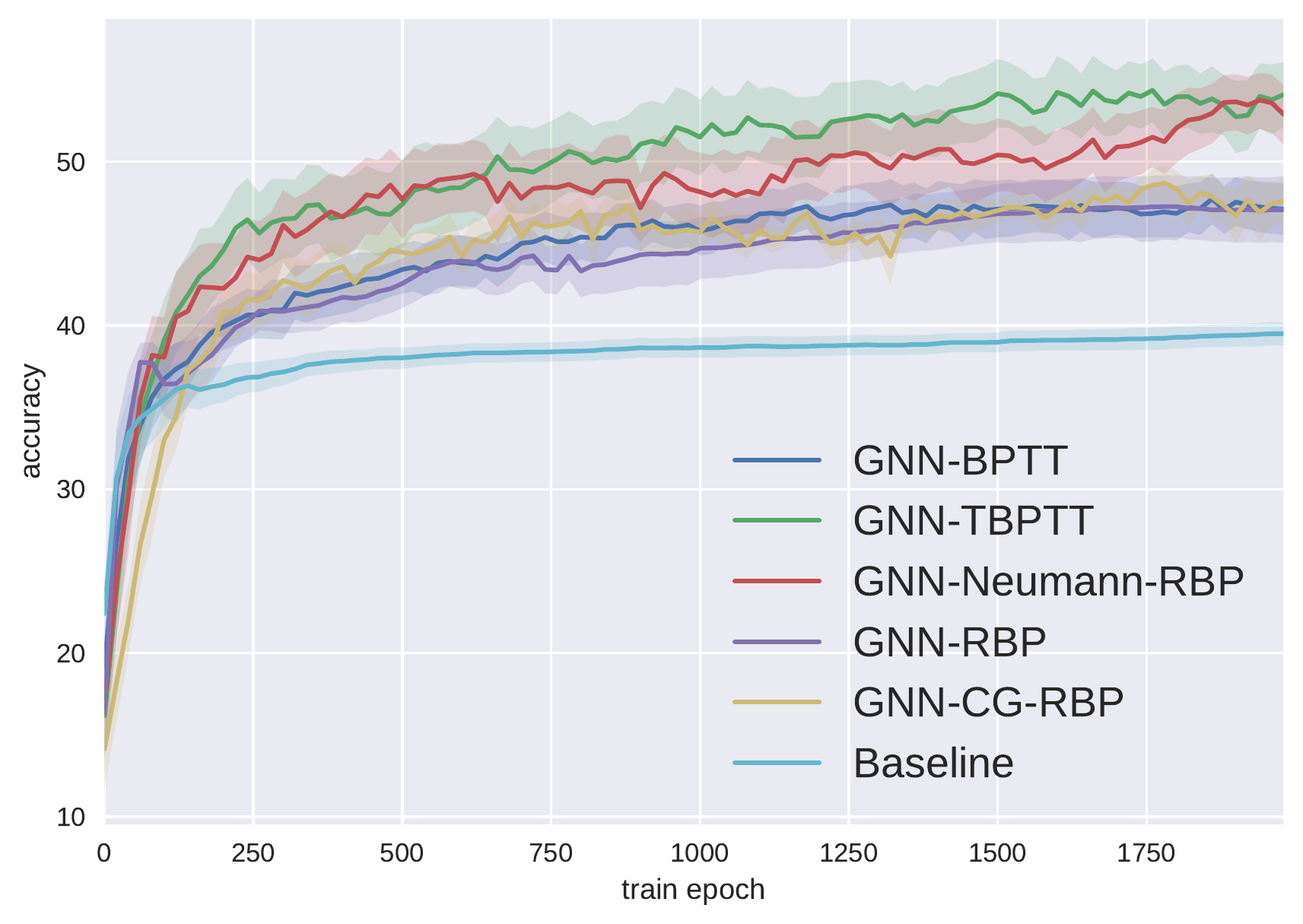}&
    \includegraphics[width=0.32\linewidth]{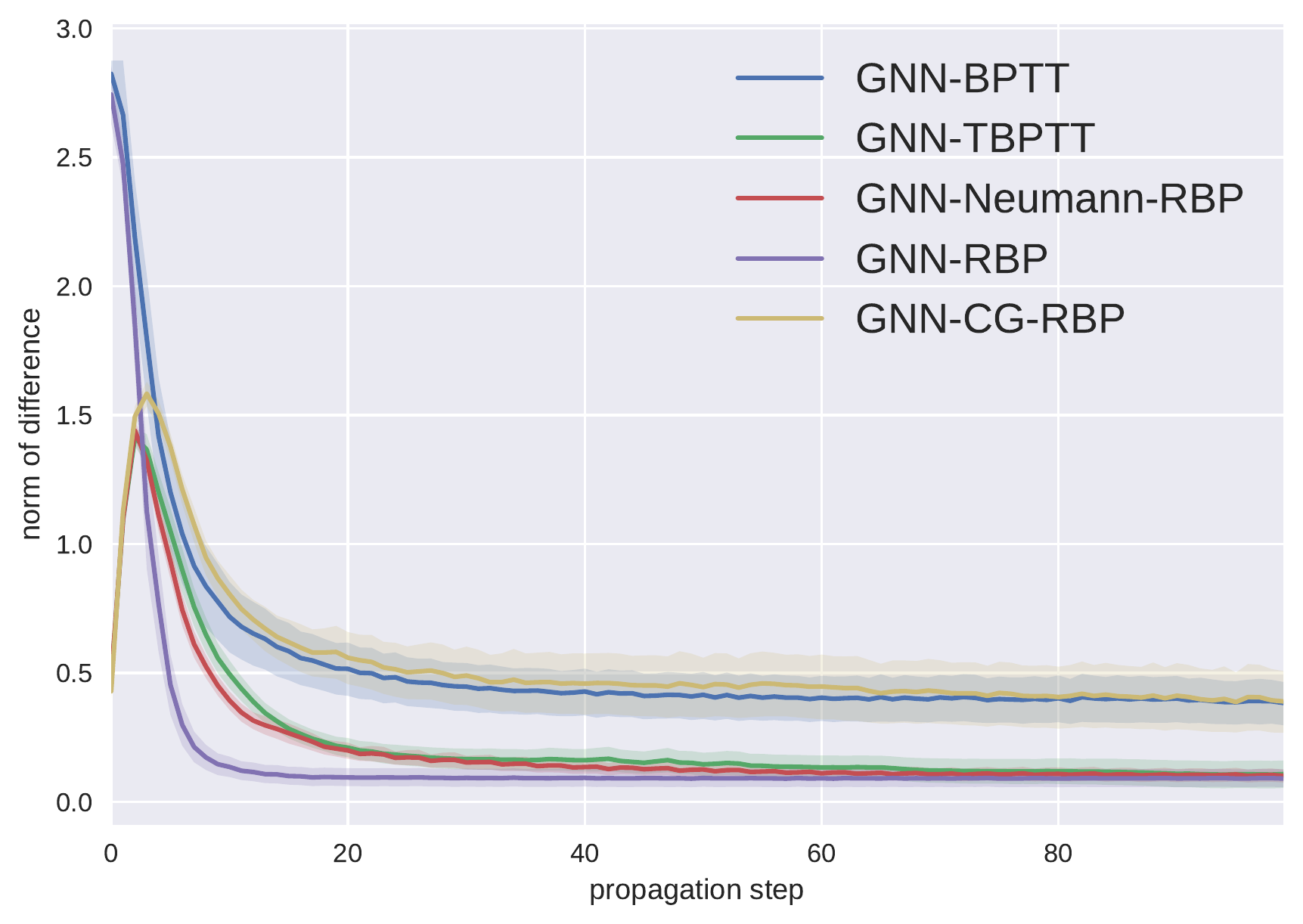}\\    
    (d) & (e) & (f) \\
\end{tabular}
\vspace{-0.3cm}
\caption{The first and second rows are the results on Cora and Pubmed respectively.
(a) to (c), (d) to (f) are curves of training loss, validation accuracy and difference norm of the same GNN with different optimization methods.} 
\vspace{-0.5cm}
\label{fig:citation_curve}
\end{figure*}

The classical testbed for RBP is the associative memory~\cite{hopfield1982neural}. 
Several images or patterns are presented to the neural network which learns to store or memorize the images.
After the learning process, the network is subsequently presented with a corrupted or noisy version of the original image. 
The task is then to retrieve the corresponding original image.
We consider a simplified continuous Hopfield network as described in~\cite{haykinneural}.
Specifically, the system of nonlinear first-order differential equations is,
\begin{align}\label{eq:hopfield}
\frac{d}{dt} h_{i}(t) = - \frac{b \cdot h_{i}(t)}{a} + \sum_{j=1}^{N} w_{ij} \phi(b \cdot h_{j}(t)) + I_{i},
\end{align}
where subscript $i$ denotes the index of the neuron. 
$w_{ij}$ is the learnable weight between a pair of neurons.
$h_{i}$ is the hidden state of the $i$-th neuron.
$\phi$ is a nonlinear activate function which is a sigmoid function in our experiments.
$a, b$ are positive constants and are set to $1$ and $0.5$.
The set of neurons consists of three parts: observed, hidden and output neurons, of size $784$, $1024$ and $784$ respectively.
For observed neuron, the state $h_{i}$ is clamped to observed pixel value $I_{i}$.
For hidden and output neurons, the observed pixel value $I_{i} = 0$ and their states $h_{i}$ are updated according to Eq.~(\ref{eq:hopfield}).
During inference, the output neurons return $x_{i} = \phi(b \cdot h_{i})$ and we further binarize it for visualization.
An important property of continuous Hopfield networks is that by updating the states according to Eq.~(\ref{eq:hopfield}) until convergence (which corresponds to the forward pass of RNNs), we are guaranteed to minimize the following (Lyapunov) energy function.
\begin{align}\label{eq:lyapunov}
E = \sum_{i=1}^{N} \left( \frac{1}{a} \int_{0}^{x_{i}} \phi^{-1}(x) dx - I_{i} x_{i} \right) - \sum_{i=1}^{N} \sum_{j=1}^{N} \frac{w_{ij} x_{i} x_{j}}{2}, \nonumber
\end{align}
where we drop the dependency on time $t$ for simplicity.
Instead of adopting the Hebbian learning rule as in~\cite{hopfield1982neural}, we directly formulate the learning objective as minimizing $\sum_{i \in \mathcal{I}} \Vert x_{i} - I_{i} \Vert_{1}$ where $\mathcal{I}$ is the set of observed neurons.
In our experiments, we train and test on $10$ MNIST images. 
In training we feed clean data, and during testing we randomly corrupt $50\%$ of the non-zero pixel values to zero. 
The number of updates for one inference pass is $50$.

Fig.~\ref{fig:hopfield_curve} shows the training and validation curves of continuous Hopfield network with different optimization methods. 
Here truncation steps for TBPTT, RBP, CG-RBP and Neumann-RBP are all set to $20$.
From the figure, we can see that CG-RBP and Neumann-RBP match BPTT under this setting which verifies that their gradients are accurate.
Nevertheless, we can see that training curve of the original RBP blows up which validates its instability issue.
The hidden state of Hopfield network becomes steady within $10$ steps.
However, we notice that if we set the truncation step to $10$, original RBP exhibits behaviors which fails to converge.
We also show some visualizations of retrieved images of the Hopfield network under different optimization methods in Fig.~\ref{fig:hopfield_vis}.
More visual results are provided in the appendix.

\subsection{Semi-supervised Document Classification}

We investigate RBPs on semi-supervised document classification with citation networks.
A node of a network represents a document associated with a bag-of-words feature.
Nodes are connected based on the citation links.
Given a portion of nodes labeled with subject categories, e.g., science, history, the task is to predict the categories for unlabeled nodes within the same network.
We use two citation networks from~\cite{yang2016revisiting}, i.e., Cora, Pubmed, of which the statistics are summarized in the appendix.
We adopt graph neural networks (GNNs)~\cite{scarselli2009graph} model and employ the GRU as the update function similarly as~\cite{li2015gated}.
We refer to~\cite{li2015gated,liao2017graph} for more details.
We compare different optimization methods with the same GNN.
We also add a logistic regression model as a baseline which is applied to every node independently.
The labeled documents are randomly split into $1\%$, $49\%$ and $50\%$ for training, validation and testing.
We run all experiments with $10$ different random seeds and report the average results.
The training, validation and difference norm curves of BPTT, TBPTT and all RBPs are shown in Fig.~\ref{fig:citation_curve}.
We can see that the hidden states of GNNs with different optimization methods become steady during inference from Fig.~\ref{fig:citation_curve} (c).
As shown in Fig.~\ref{fig:citation_curve} (a) and (b), Neumann-RBP is on par with TBPTT on both datasets. 
This matches our analysis in proposition~\ref{prop:neumann_rbp_1} since the changes of successive hidden states of TBPTT and Neumann-RBP are almost zero as shown in Fig.~\ref{fig:citation_curve} (c).
Moreover, they outperform other variants and the baseline model. 
On the other hand, BPTT on both datasets encounter issues in learning which may be attributable to the accumulation of errors in the many steps of unrolling.
Note that CG-RBP sometimes performs significantly worse than Neumann-RBP, e.g., on Cora. 
This may be caused by the fact that the underlying linear system of CG-RBP is ill-conditioned in some applications as the condition number is squared in CGNE.
The test accuracy of different methods are summarized in Table~\ref{table:citation}.
It generally matches the behavior in the validation curves.


\begin{table}
\centering
\begin{tabular}{@{}c|cc@{}}
\hline
\toprule
Test Acc. & Cora & Pubmed \\
\midrule
\midrule
Baseline & 39.96 $\pm$ 3.4 & 40.41 $\pm$ 3.1 \\
BPTT & 24.48 $\pm$ 6.6 & 47.05 $\pm$ 3.1 \\
TBPTT & 46.55 $\pm$ 6.4 & 53.41 $\pm$ 6.7 \\
RBP & 29.25 $\pm$ 3.3 & 48.55 $\pm$ 3.4 \\
CG-RBP & 39.26 $\pm$ 6.5 & 49.12 $\pm$ 2.9 \\
Neumann-RBP & \textbf{46.63} $\pm$ 8.3 & \textbf{53.56} $\pm$ 5.3 \\
\bottomrule
\end{tabular}
\vspace{-0.3cm}
\caption{Test accuracy of different methods on citation networks.} 
\vspace{-0.1cm}
\label{table:citation}
\end{table}

\subsection{Hyperparameter Optimization}\label{sect:exp_hyper_grad}

\begin{table}
\centering
\begin{tabular}{@{}c|ccc@{}}
\hline
\toprule
Truncate Step & 10 & 50 & 100 \\
\midrule
\midrule
Run Time & $\times$3.02 & $\times$2.87 & $\times$2.68 \\
Memory & $\times$4.35 & $\times$4.25 & $\times$4.11 \\
\bottomrule
\end{tabular}
\vspace{-0.3cm}
\caption{Run time and memory comparison. We show the ratio of BPTT's cost divided by Neumann-RBP's.}
\vspace{-0.4cm}
\label{table:rbp_runtime}
\end{table}

\begin{figure}[t]
\centering
\begin{tabular}{c@{\hspace{0.1mm}}c}
    \includegraphics[width=0.49\linewidth]{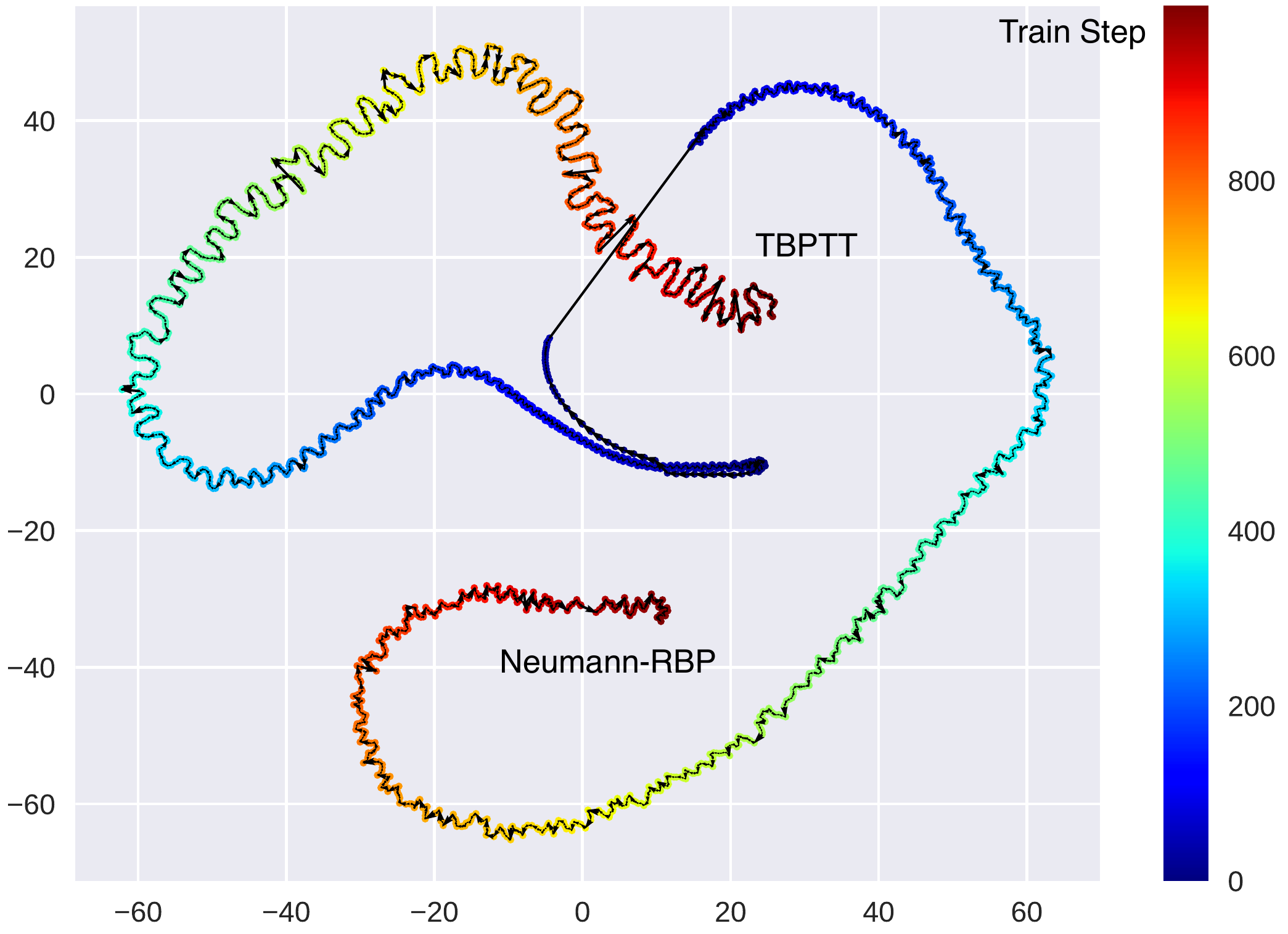}&
    \includegraphics[width=0.49\linewidth]{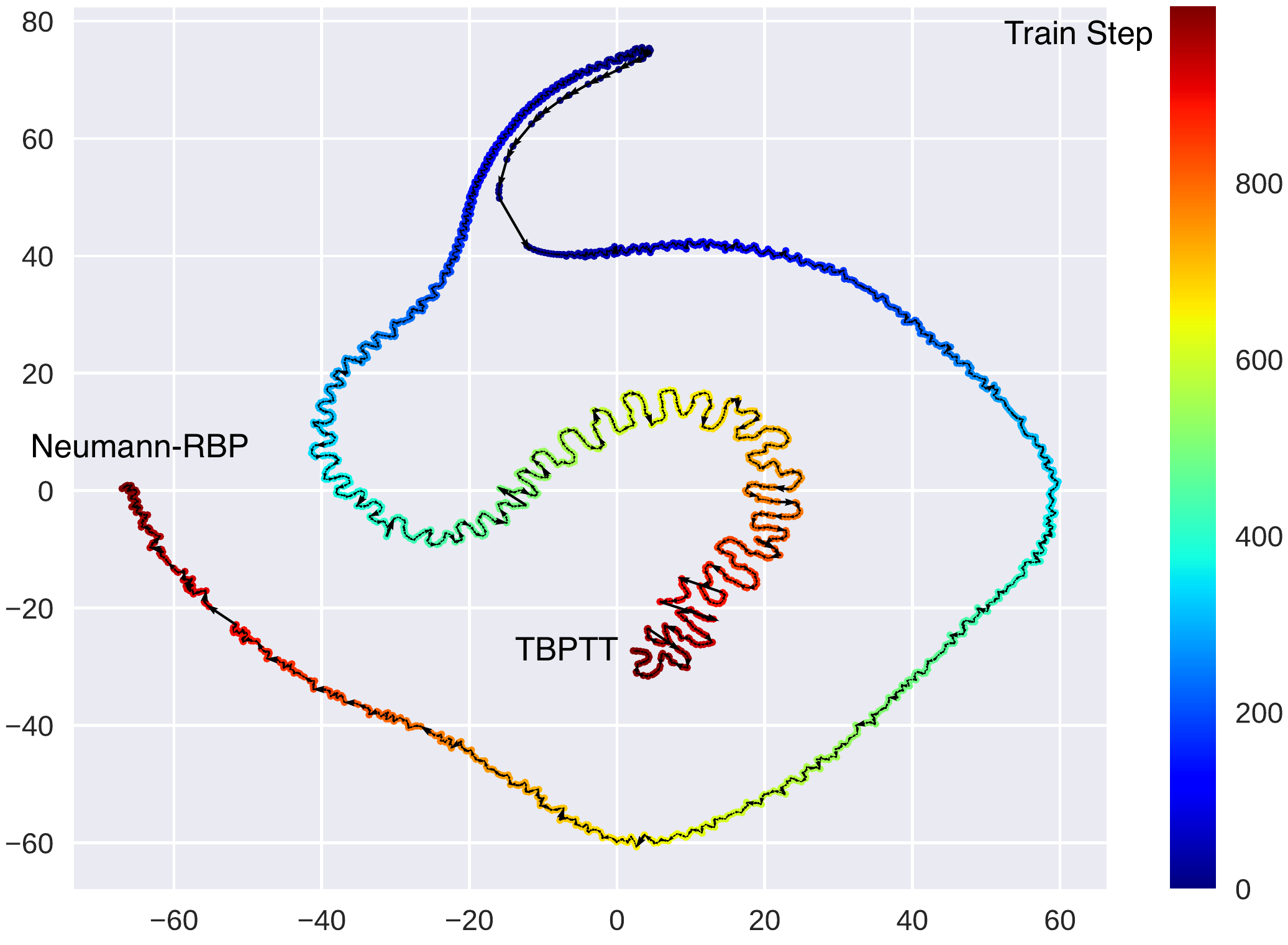}\\
    (a) & (b) \\
\end{tabular}
\vspace{-0.3cm}
\caption{t-SNE visualization of trajectories of hidden states for TBPTT and Neumann-RBP on hyperparameter optimization. Different methods are annotated at the convergence point. (a) and (b) are snapshots of hidden states at meta step $20$ and $40$. Time is encoded as the color map for better illustration.}
\vspace{-0.5cm}
\label{fig:vis_trajectory}
\end{figure}

\begin{figure*}[t]
\renewcommand*{\arraystretch}{0.1}
\centering
\begin{tabular}{@{\hspace{0.1mm}}c@{\hspace{0.1mm}}c@{\hspace{0.1mm}}c@{\hspace{0.1mm}}c}
    \includegraphics[width=0.245\linewidth]{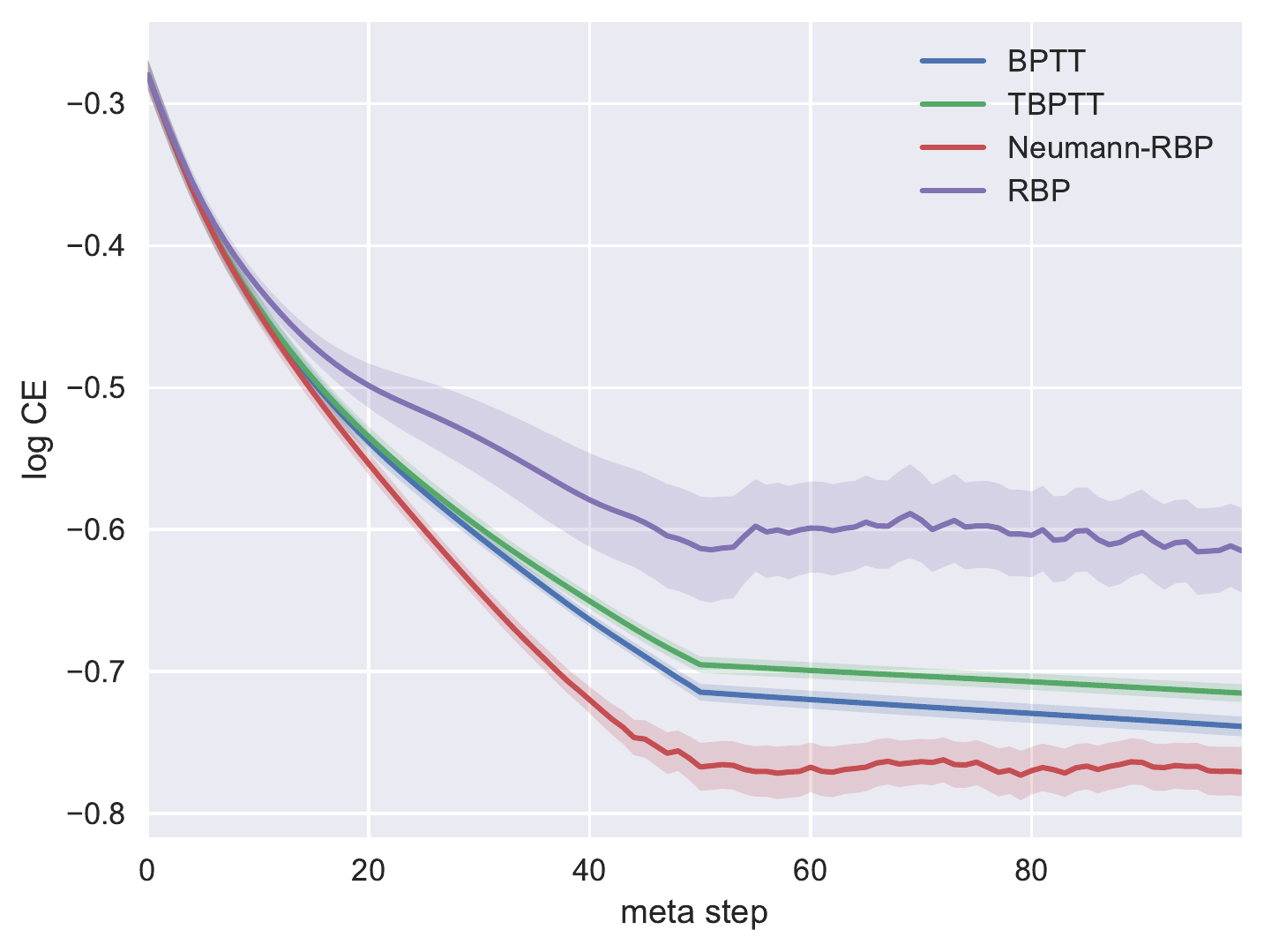}&
    \includegraphics[width=0.245\linewidth]{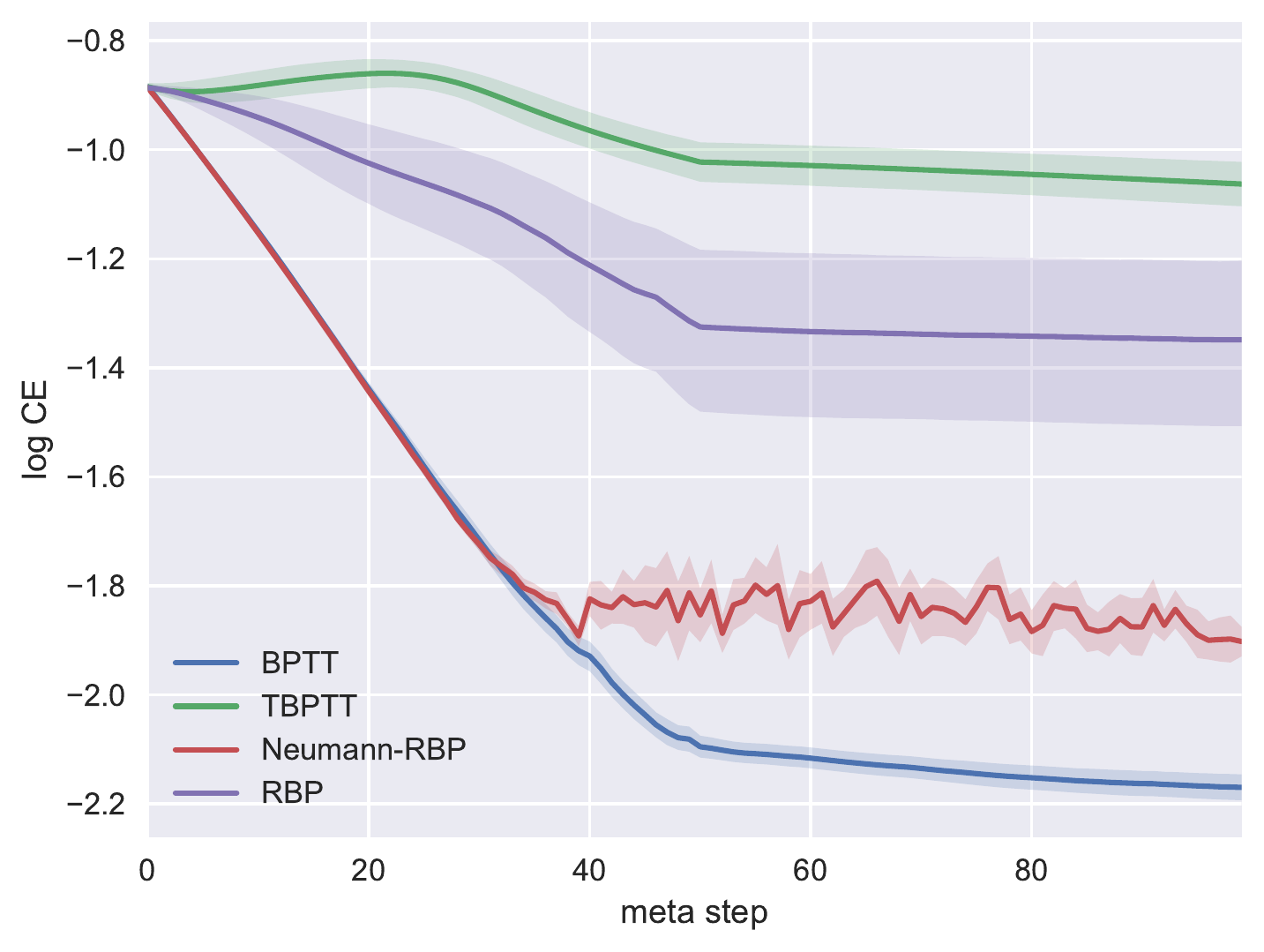}&
    \includegraphics[width=0.245\linewidth]{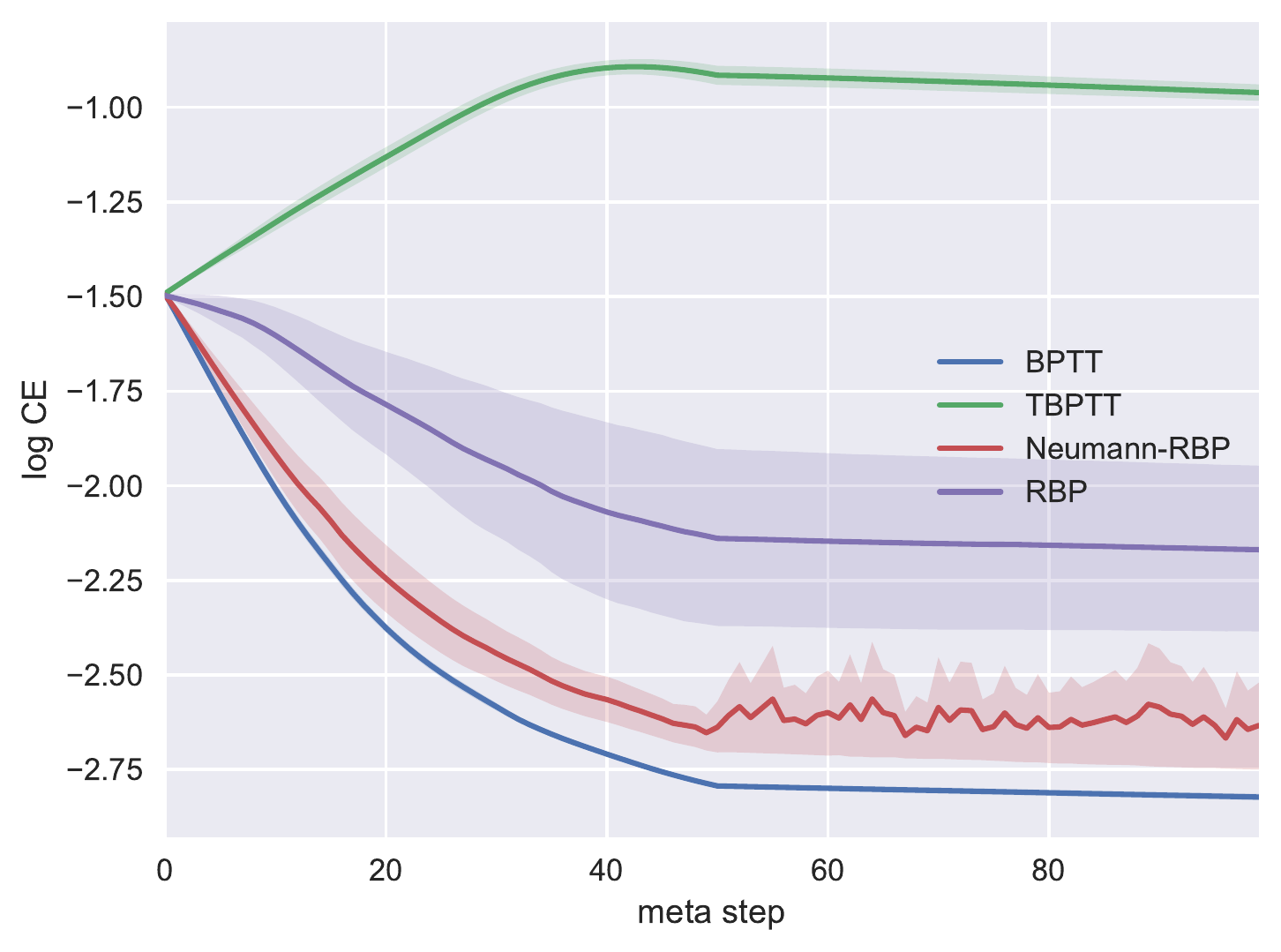}&
    \includegraphics[width=0.245\linewidth]{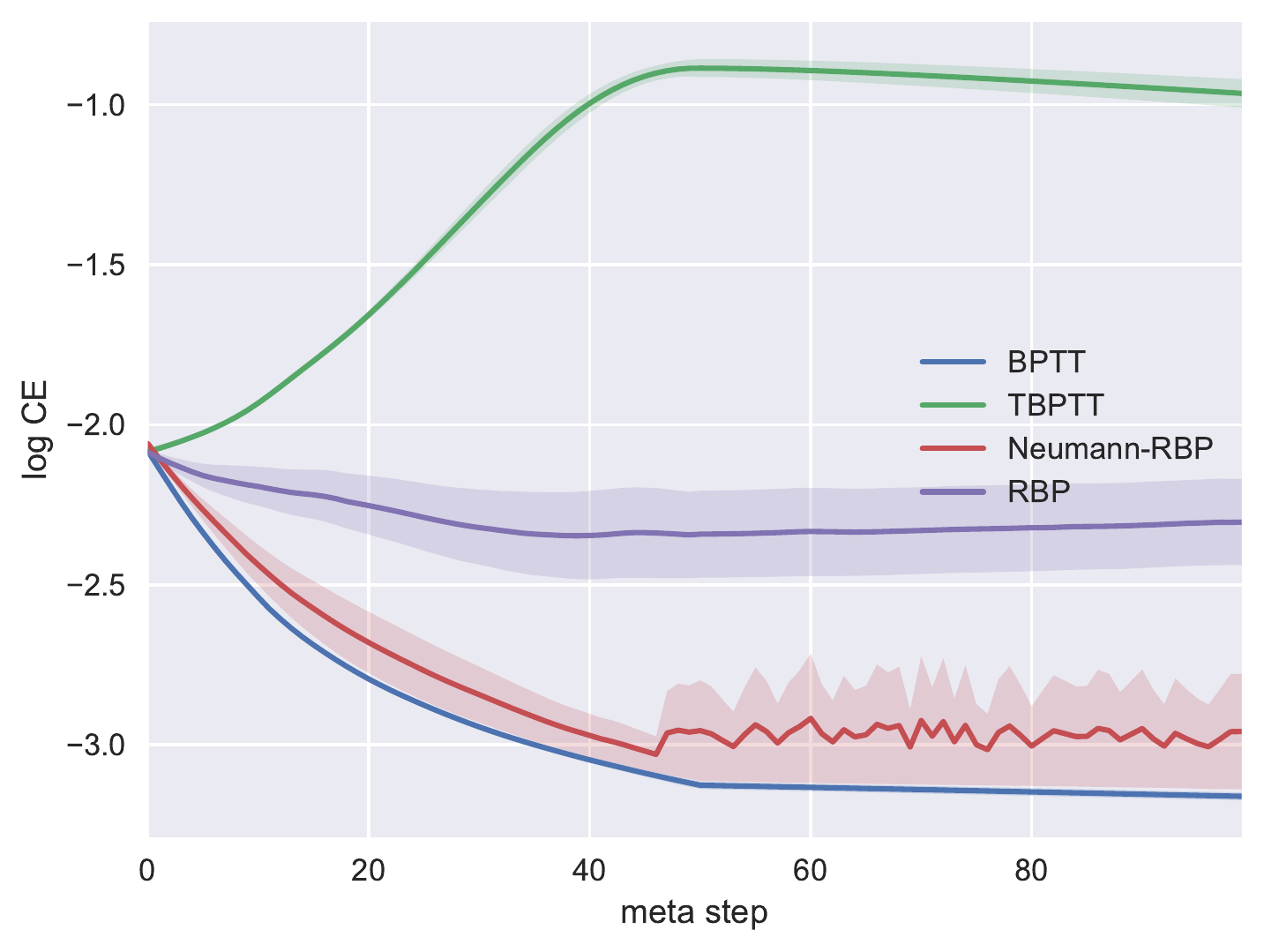}\\
    (a) & (b) & (c) & (d) \\
    \includegraphics[width=0.245\linewidth]{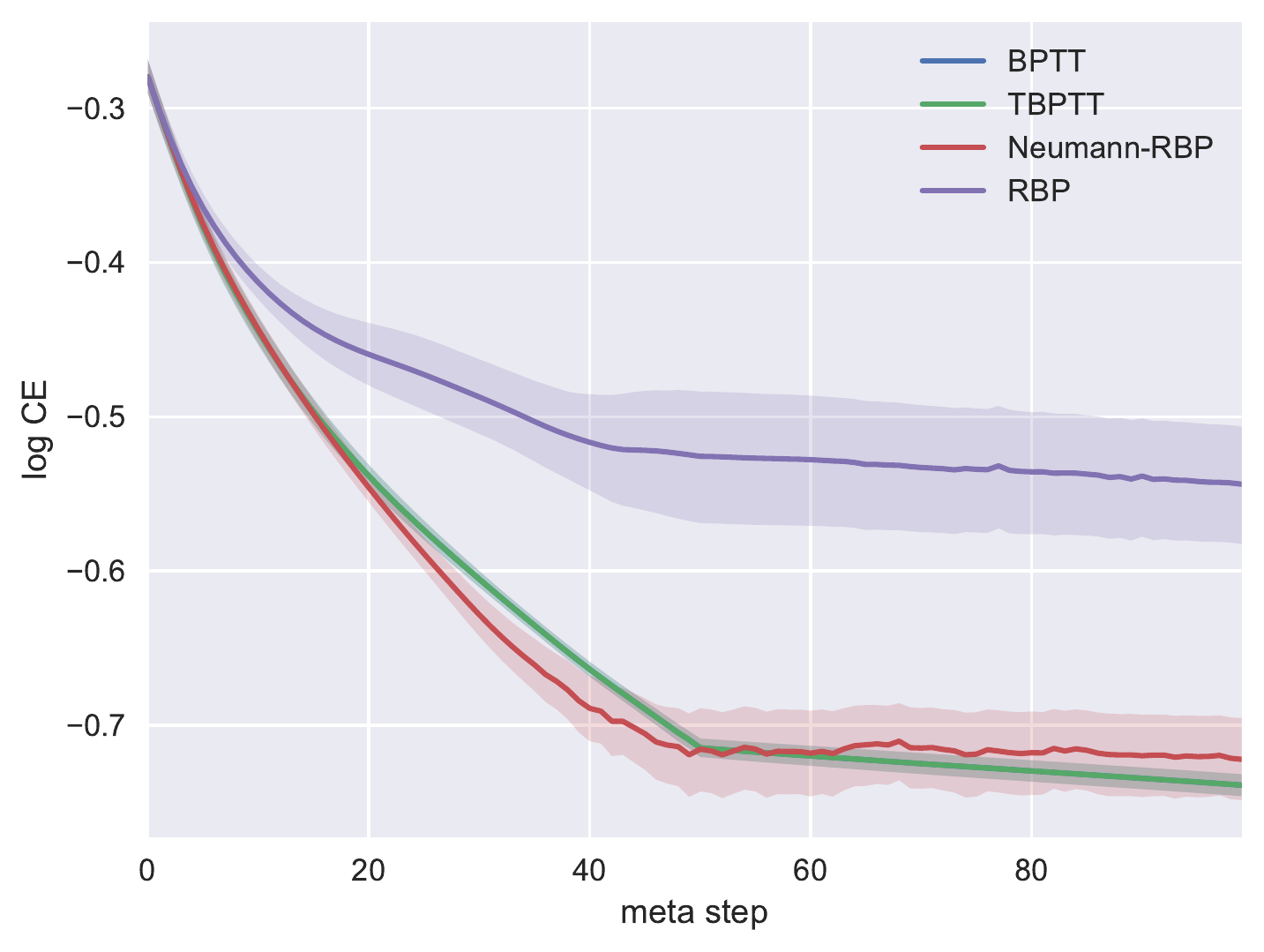}&
    \includegraphics[width=0.245\linewidth]{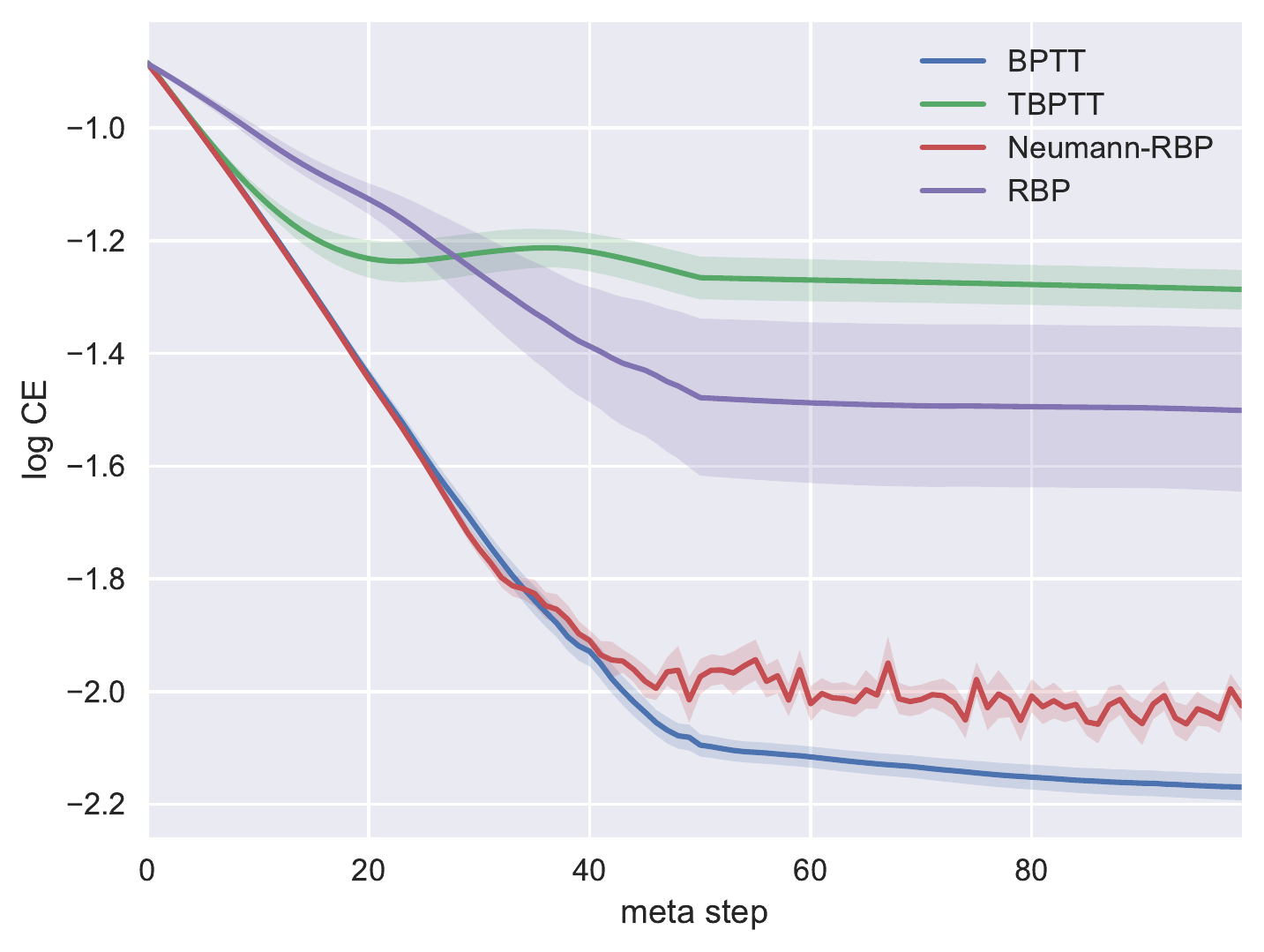}&
    \includegraphics[width=0.245\linewidth]{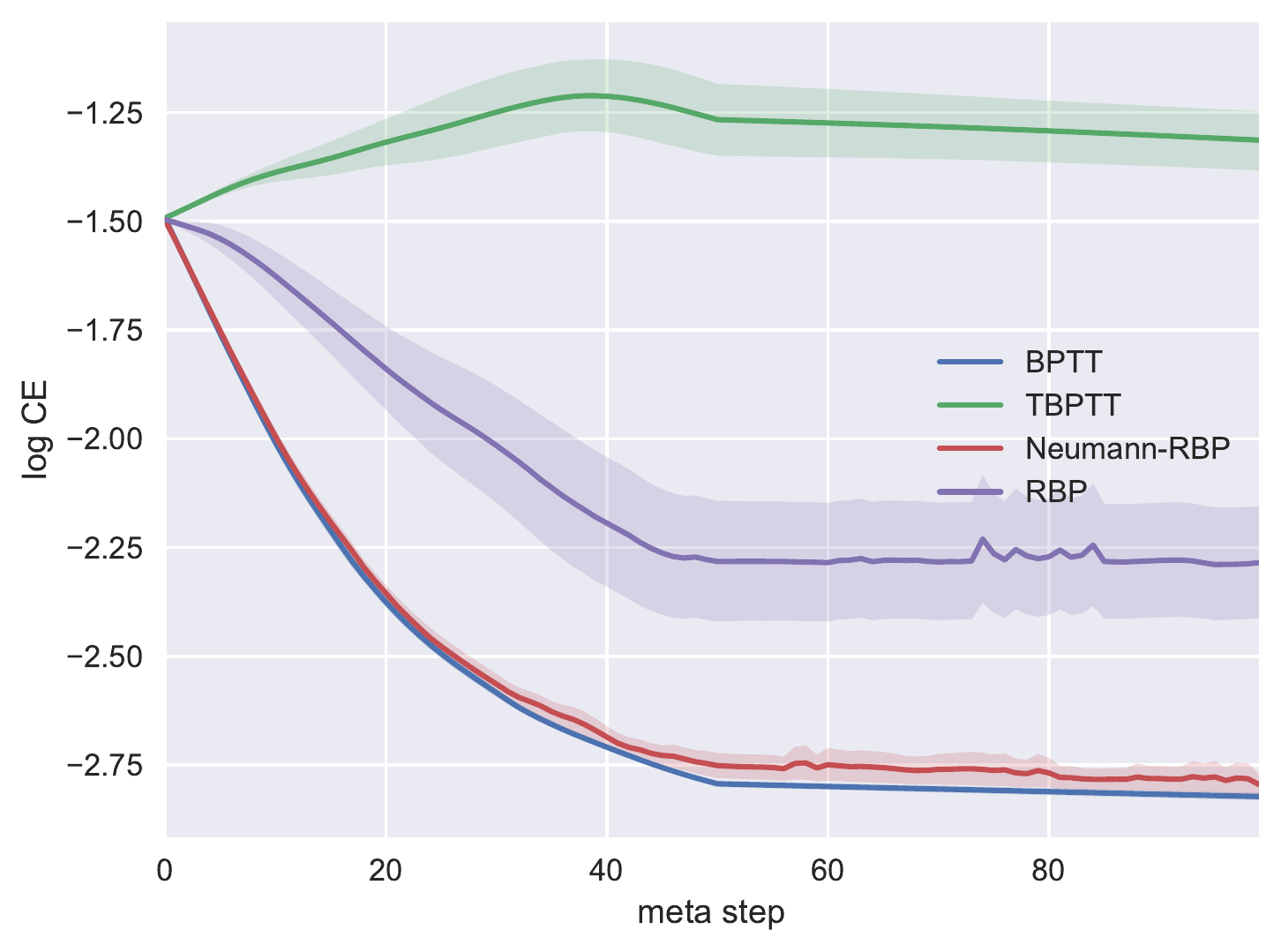}&
    \includegraphics[width=0.245\linewidth]{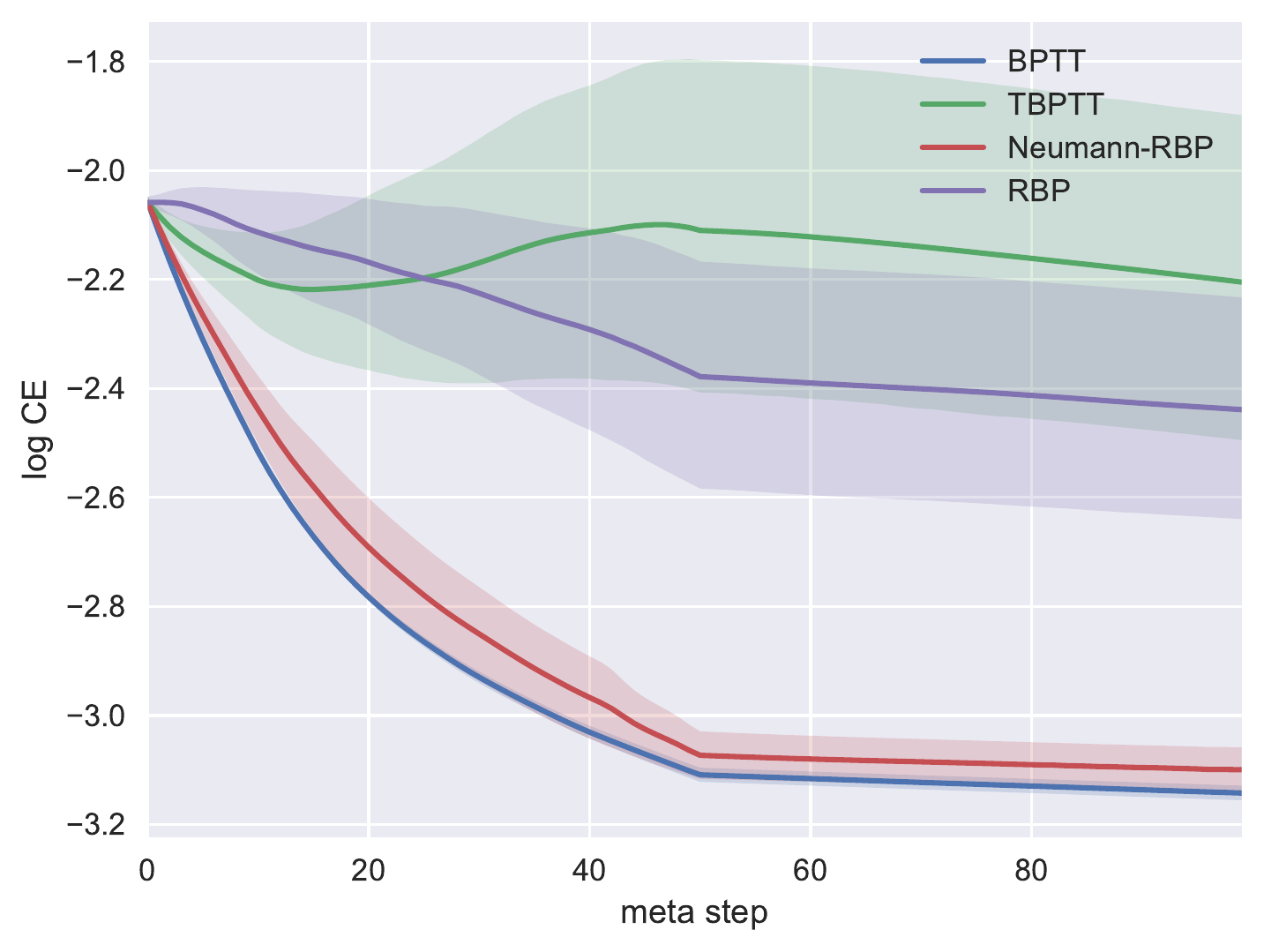}\\
    (e) & (f) & (g) & (h) \\  
\end{tabular}
\vspace{-0.3cm}
\caption{Meta training loss. Training and truncate steps per meta step are (a) (100, 50); (b) (500, 50); (c) (1000, 50); (d) (1500, 50); (e) (100, 100); (f) (500, 100); (g) (1000, 100); (h) (1500, 100).} 
\vspace{-0.3cm}
\label{fig:hyper_grad_meta_loss}
\end{figure*}

\begin{figure*}[t]
\renewcommand*{\arraystretch}{0.1}
\centering
\begin{tabular}{@{\hspace{0.1mm}}c@{\hspace{0.1mm}}c@{\hspace{0.1mm}}c@{\hspace{0.1mm}}c}
    \includegraphics[width=0.245\linewidth]{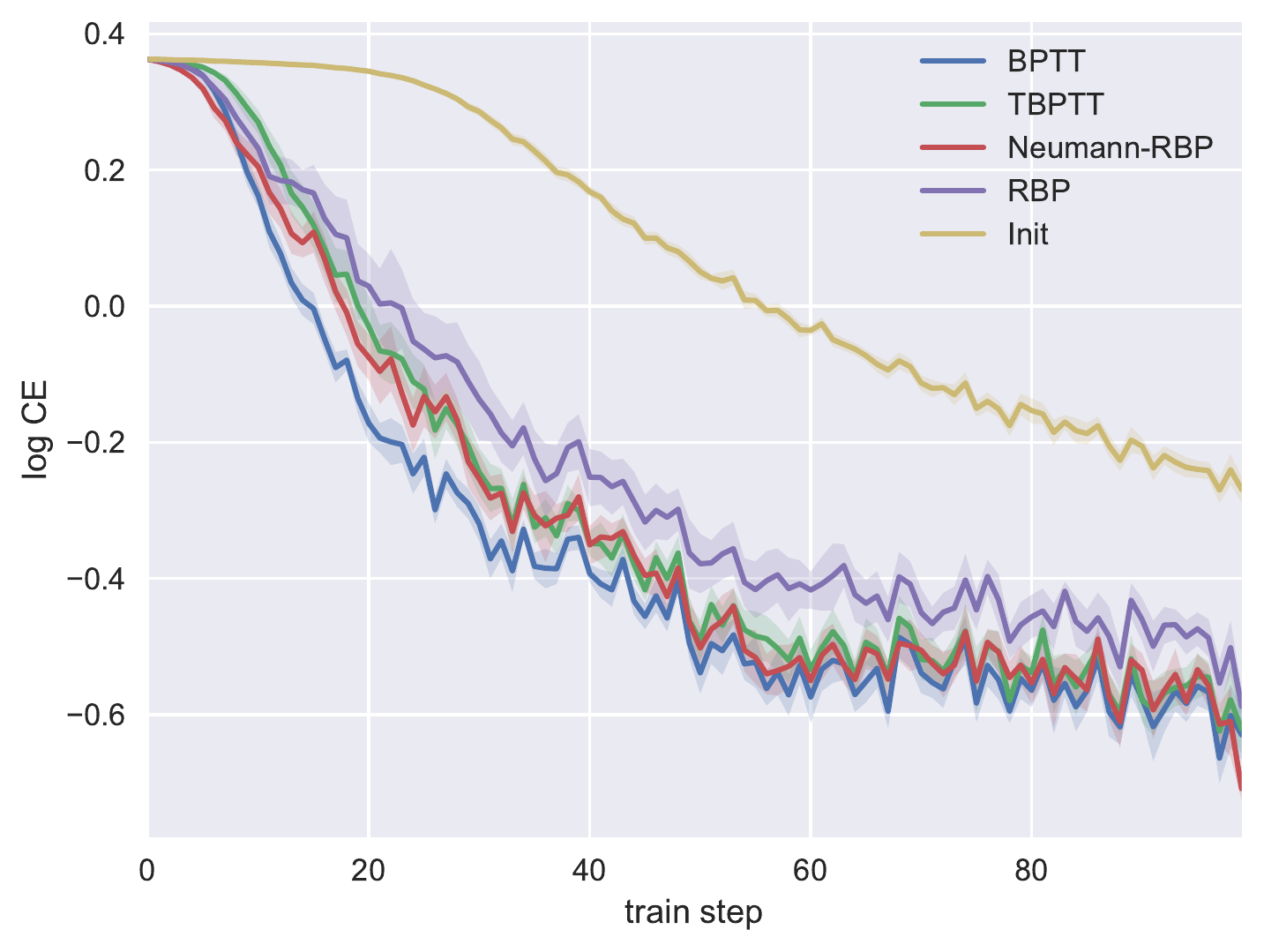}&
    \includegraphics[width=0.245\linewidth]{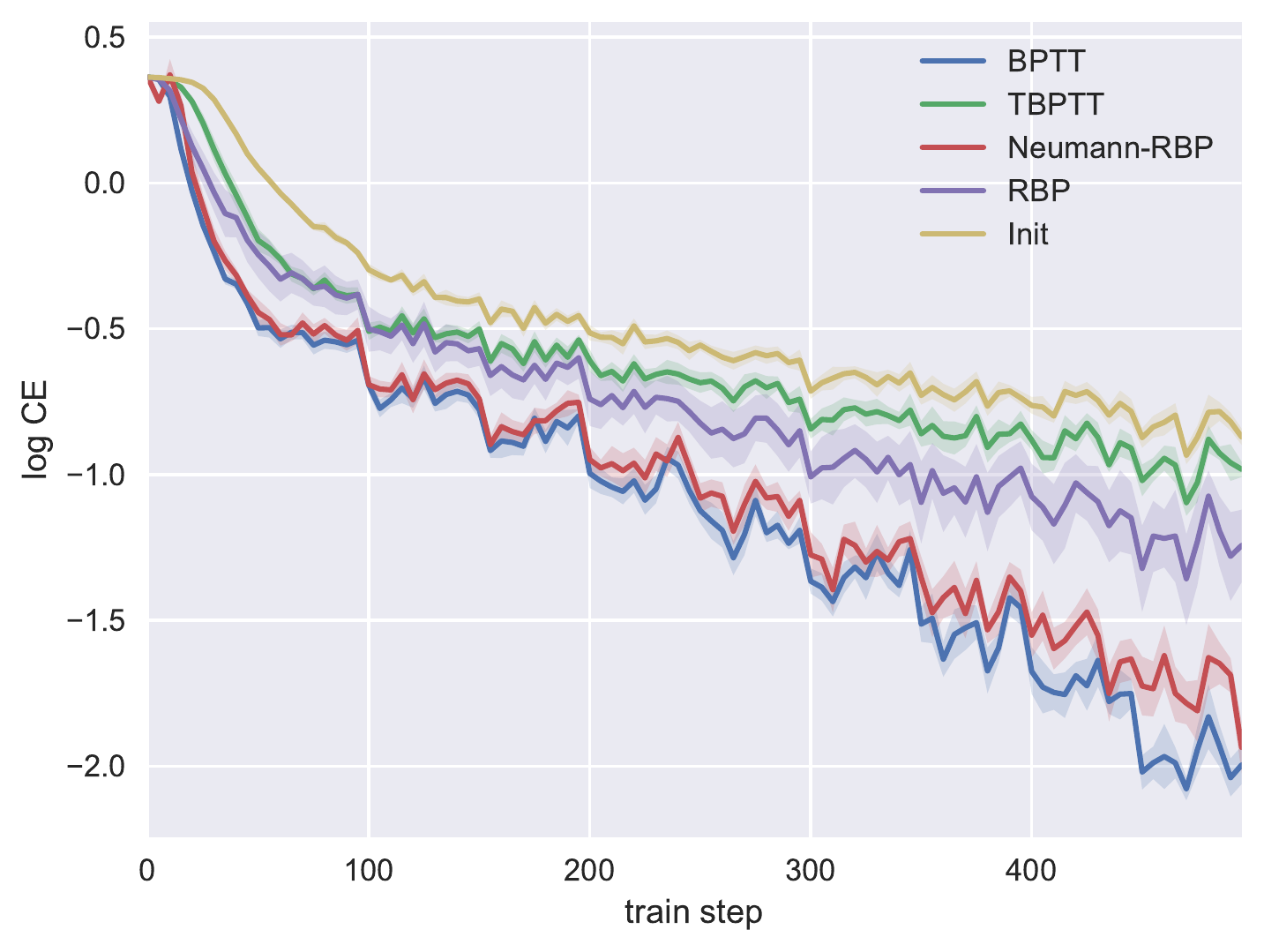}&
    \includegraphics[width=0.245\linewidth]{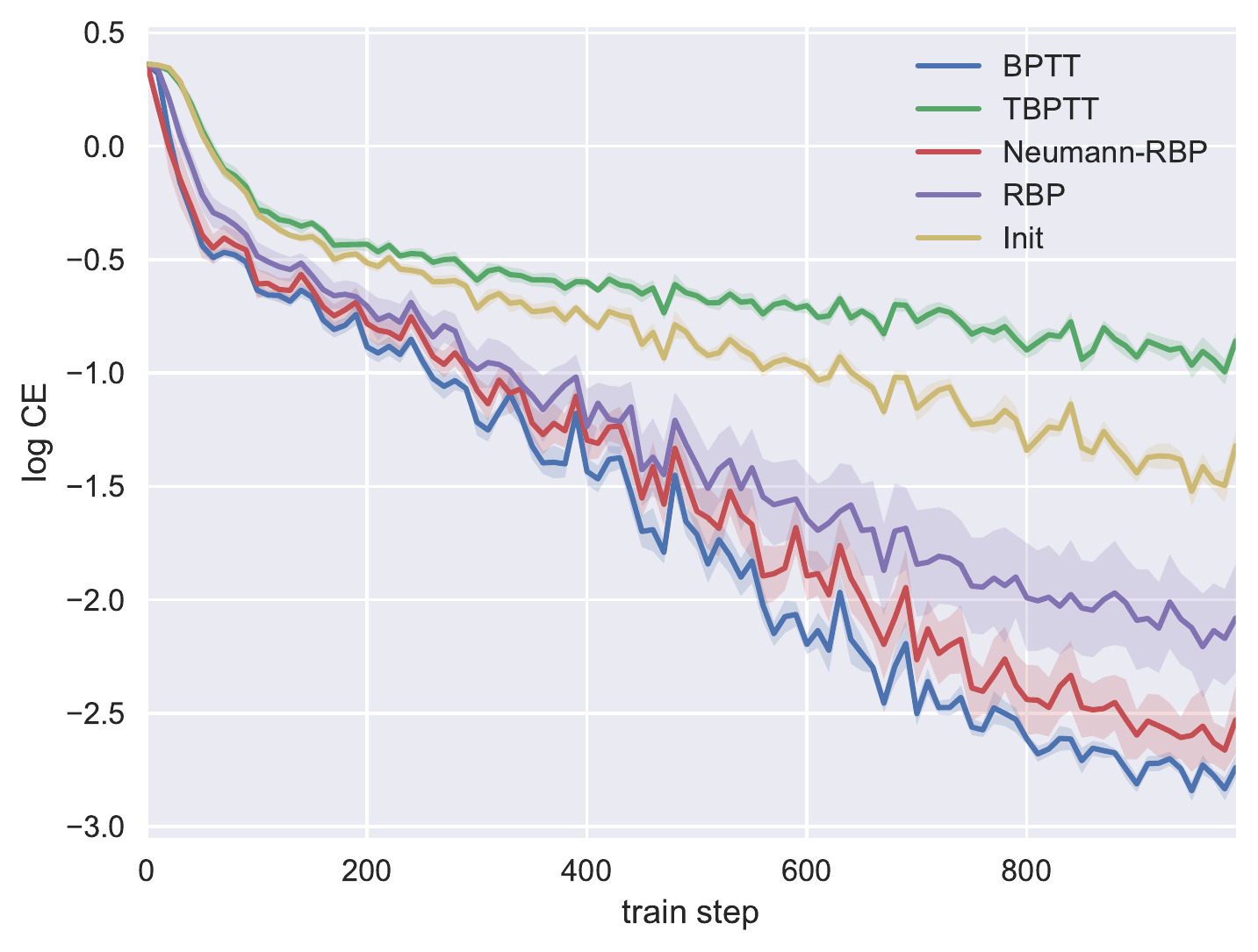}&
    \includegraphics[width=0.245\linewidth]{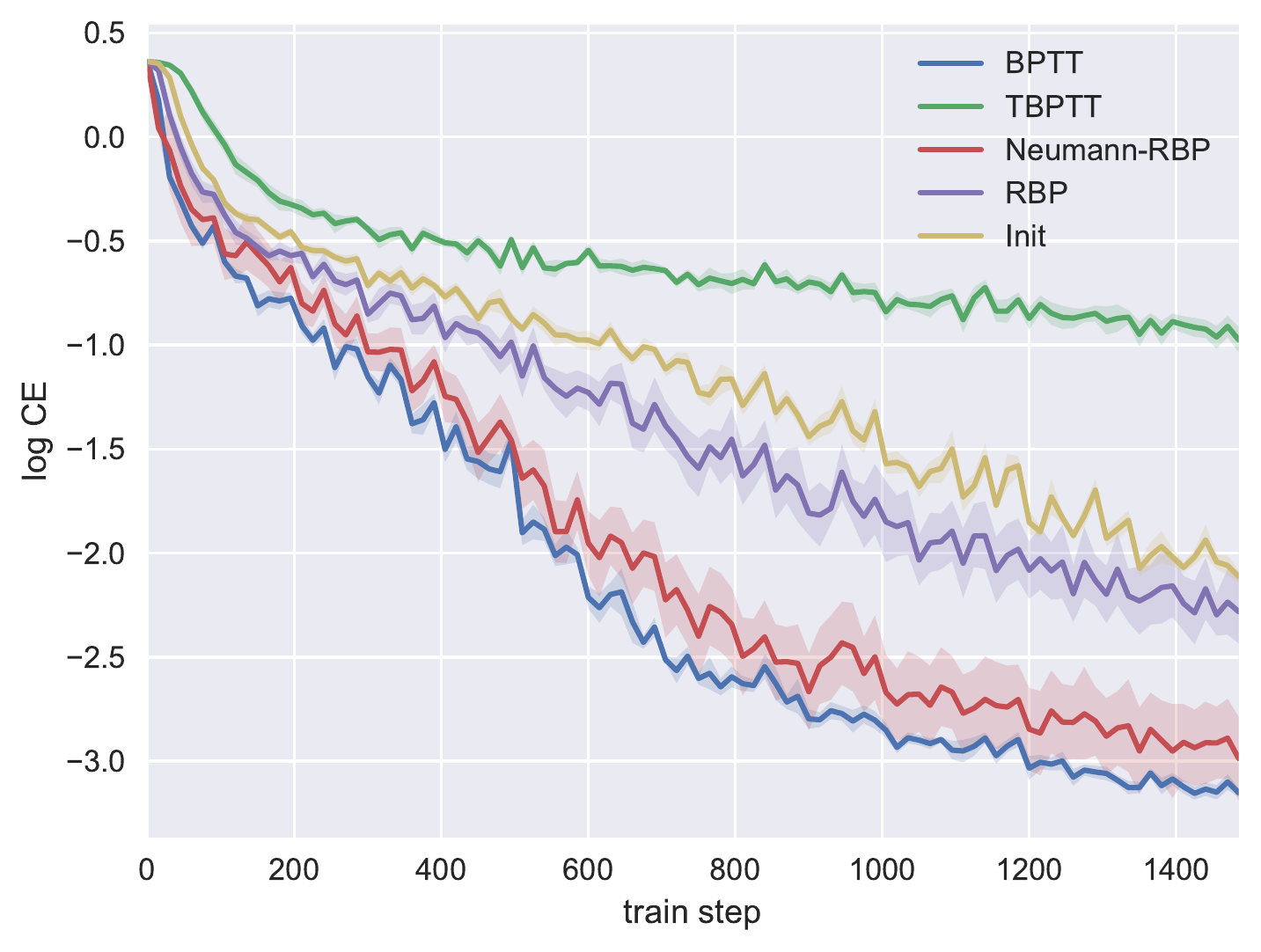}\\
    (a) & (b) & (c) & (d) \\
    \includegraphics[width=0.245\linewidth]{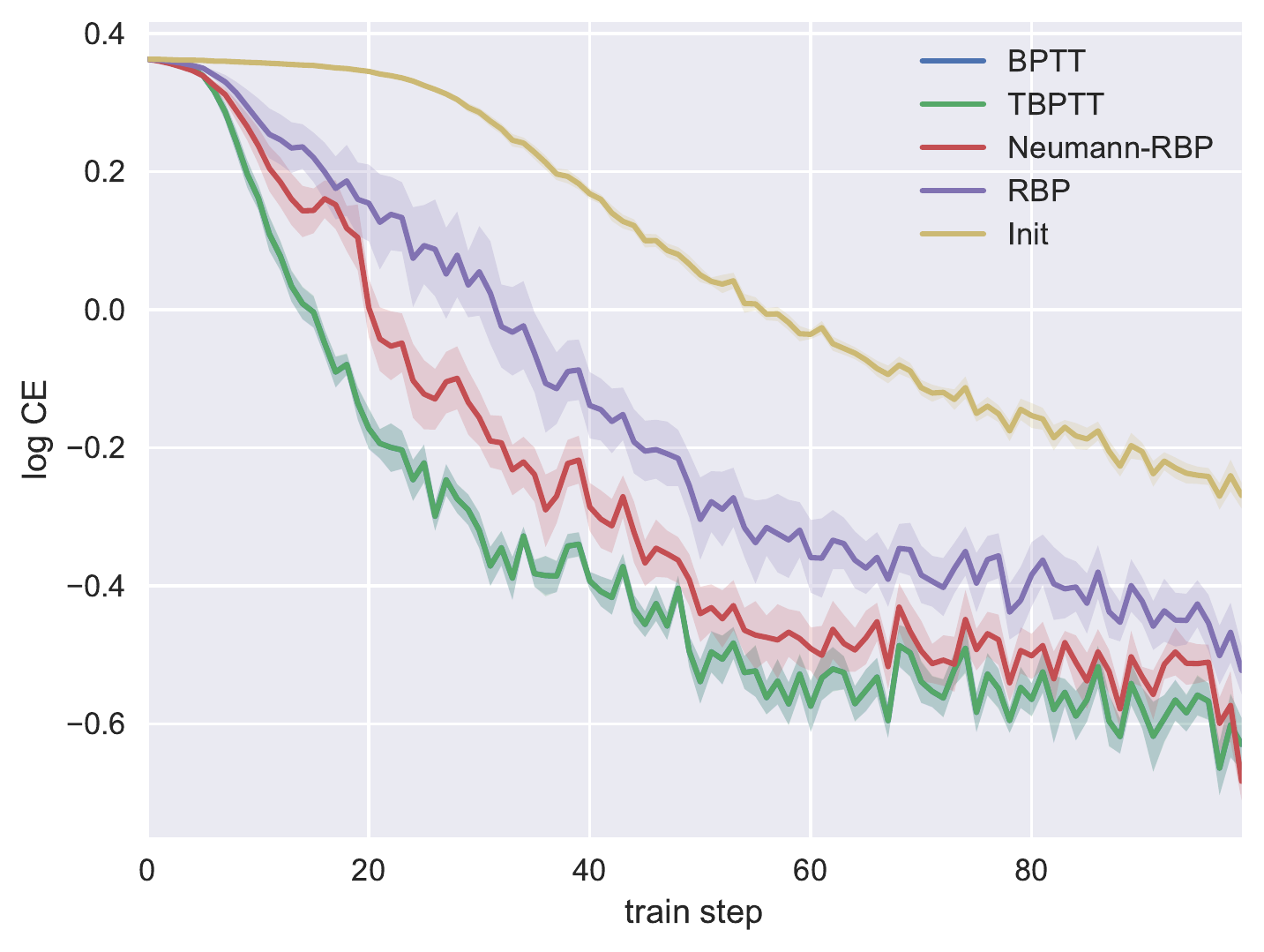}&
    \includegraphics[width=0.245\linewidth]{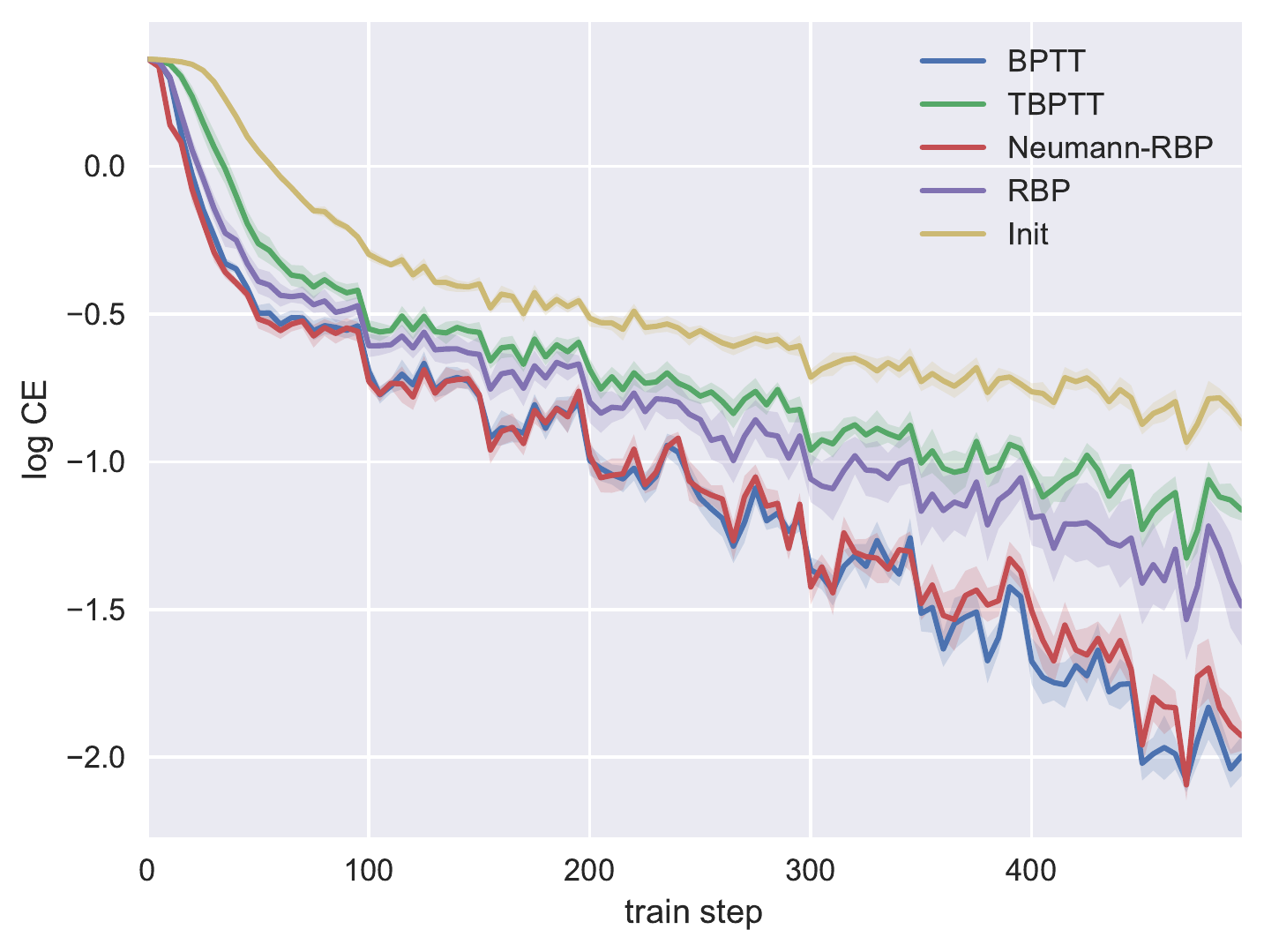}&
    \includegraphics[width=0.245\linewidth]{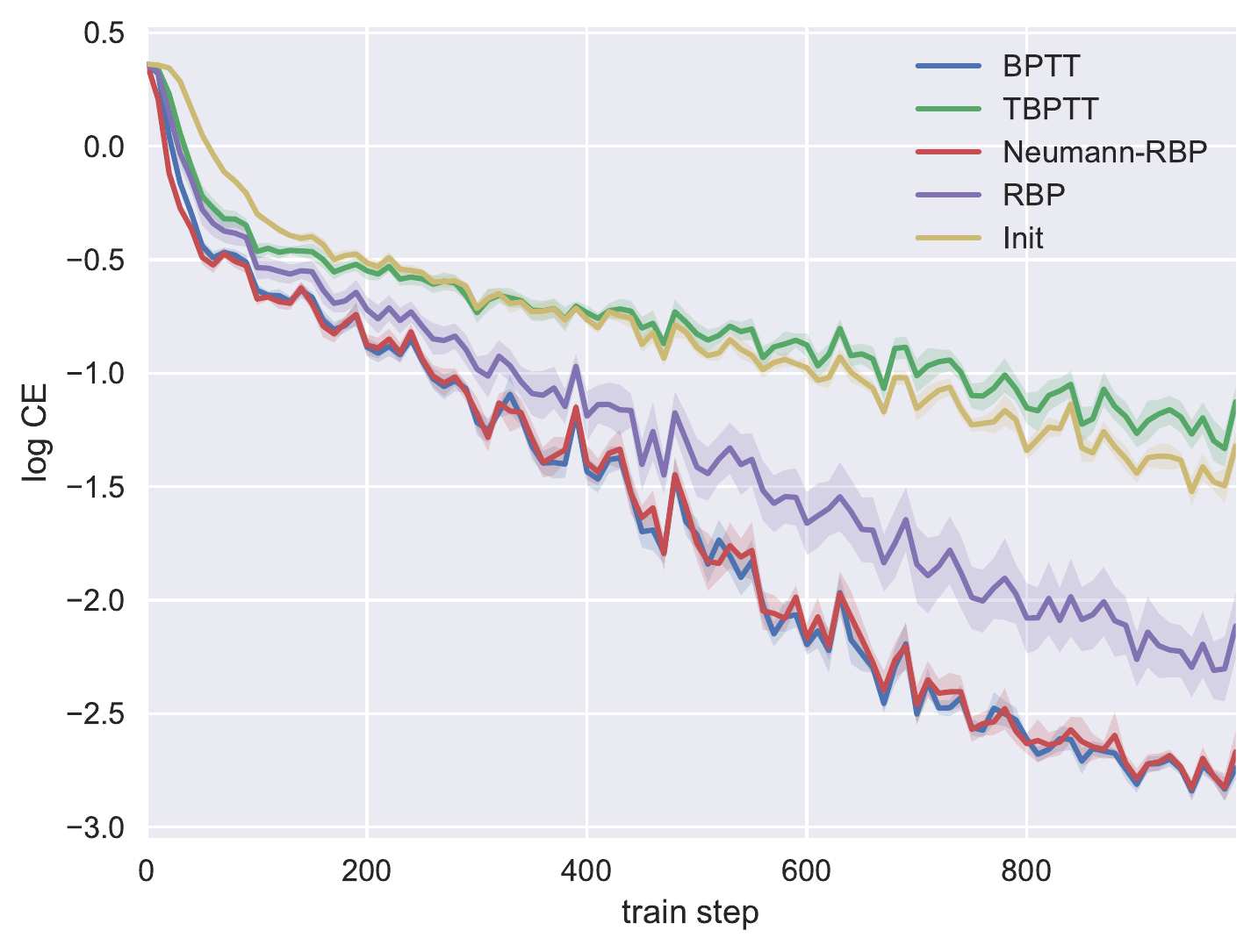}&
    \includegraphics[width=0.245\linewidth]{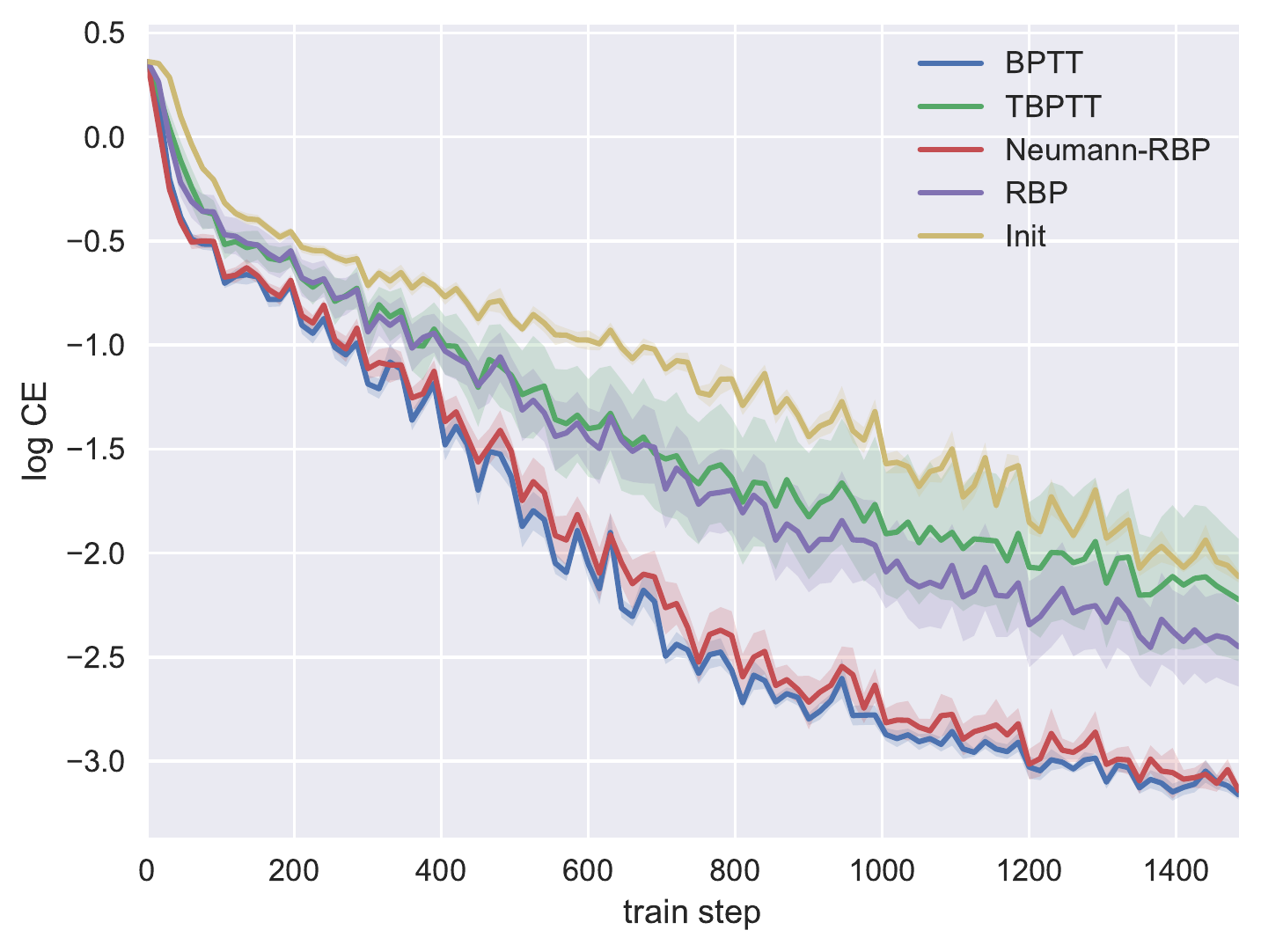}\\
    (e) & (f) & (g) & (h) \\
\end{tabular}
\vspace{-0.3cm}
\caption{Training loss at last meta step. Training and truncate steps per meta step are (a) (100, 50); (b) (500, 50); (c) (1000, 50); (d) (1500, 50); (e) (100, 100); (f) (500, 100); (g) (1000, 100); (h) (1500, 100).} 
\vspace{-0.5cm}
\label{fig:hyper_grad_train_loss}
\end{figure*}

In our next experiment, we test the abilities of RBP to perform hyperparameter optimization. 
In this experiment, we view the optimization process as a RNN.
When training a neural network, the model parameters, e.g., weights and bias, are regarded as the hidden states of the RNN.
Hyperparameters such as learning rate and momentum are learnable parameters of this `meta-learning' RNN. 
Here we focus on the gradient based hyperparameter optimization rather than the gradient-free one~\cite{snoek2012practical}.
We adopt the same experiment setting as in~\cite{maclaurin2015gradient}, using an initial learning rate of $\exp(-1)$ and momentum $0.5$.
The optimization is on a fully connected network with 4 layers, of sizes $784$, $50$, $50$, and $50$. 
For each layer, we associate one learning rate and one momentum with weight and bias respectively which results in $16$ hyperparameters in total. 
We use tanh non-linearities and train on $10,000$ examples on MNIST.
At each forward step of the RNN, i.e., at each optimization step, a different mini-batch of images is fed to the model. 
This is different from the previous setting where input data is fixed.
However, since the mini-batches are assumed to be i.i.d., the sequential input data can be viewed as sampled from a stationary distribution.  
We can thus safely apply RBP as the steady state holds in expectation.
In terms of implementation, we just need to average the meta gradient returned by RBPs or TBPTT across multiple mini-batches at the end of one meta step.
We use Adam~\cite{kingma2014adam} as the meta optimizer and set the learning rate to $0.05$.
The initialization of the fully connected network at each meta step is controlled to be the same.
For each hyper-gradient method, we run experiments $10$ times with $10$ different random seeds which are shared across different methods. 

Fig.~\ref{fig:hyper_grad_meta_loss} shows the meta training losses under different training and truncation steps.
For better understanding, one can consider the training step as the unrolling step of the RNN. 
Truncation step is the the number of steps that TBPTT and RBPs execute.
From the figure, we can see that as the number of training steps increases (e.g., from (a) to (d)), the meta loss becomes smoother.
This makes sense since more steps make the training per meta step closer to convergence.
Another surprising phenomenon we found is the meta loss of TBPTT becomes worse when the training step increases.
One possible explanation is that the initial meta training loss of small training steps (e.g., (a)) is still very high as you can see from the log y-axis whereas the one with large training step, e.g., (d) is much lower.
The probability of using incorrect gradients to decrease the meta loss in case (a) is most likely higher than that of (d) since it is farther from convergence.
On the other hand, our Neumann-RBP performs consistently better than the original RBP and TBPTT which empirically validates that Neumann-RBP provides better estimation of the gradient in this case.
The potential reason why RBP performs poorly is that the stochasticity of mini-batches worsen the instability issue.
Training losses under similar settings at the last meta step are also provided in Fig.~\ref{fig:hyper_grad_train_loss}.
We can see that at the end of hyperparameter optimization, our Neumann-RBP generally matches the performance of BPTT and outperforms the other methods.
Fig.~\ref{fig:vis_trajectory} depicts the trajectories of hidden states in 2D space via t-SNE~\cite{maaten2008visualizing}.
From the figure, we can see that as the meta training goes on, TBPTT tends to oscillate whereas Neumann-RBP converges, which matches the finding in the train loss curves in Fig.~\ref{fig:hyper_grad_train_loss}.

We also compare the running time and memory cost of our unoptimized Neumann-RBP implementation with the standard BPTT, i.e., using autograd of PyTorch.
With $1000$ training steps, one meta step of BPTT cost $310.4$s and $4061$MB GPU memory in average.
We take BPTT as the reference cost and report the ratio BPTT's cost divided by Neumann-RBP's in Table~\ref{table:rbp_runtime}.
All results are reported as the average of $10$ runs.
Even without optimizing the code, the practical runtime and memory footprint advantages of Neumann-RBP over BPTT is still significant.

\section{Conclusion}\label{sect:con}
In this paper, we revisit the RBP algorithm and discuss its assumptions and how to satisfy them for deep learning.
Moreover, we propose two variants of RBP based on conjugate gradient on normal equations and Neumann series.
Connections between Neumann-RBP and TBPTT are established which sheds some light on analyzing the approximation quality of the gradient of TBPTT.
Experimental results on diverse tasks demonstrate that Neumann-RBP is a stable and efficient alternative to original RBP and is promising for several practical problems.
In the future, we would like to explore RBP on hyperparameter optimization with large scale deep neural networks.

\section*{Acknowledgements}
We thank Barak Pearlmutter for the enlightening discussion and anonymous ICML reviewers for valuable comments.
R.L. was supported by Connaught International Scholarships.
R.L., E.F., L.Z., K.Y., X.P., R.U. and R.Z. were supported in part by the Intelligence Advanced Research Projects Activity (IARPA) via Department of Interior/Interior Business Center (DoI/IBC) contract number D16PC00003. 
K.Y. and X.P. were supported in part by BRAIN Initiative grant NIH 5U01NS094368. 
The U.S. Government is authorized to reproduce and distribute reprints for Governmental purposes notwithstanding any copyright annotation thereon. Disclaimer: the views and conclusions contained herein are those of the authors and should not be interpreted as necessarily representing the official policies or endorsements, either expressed or implied, of IARPA, DoI/IBC, or the U.S. Government.


\bibliography{rbp}
\bibliographystyle{icml2018}

\end{document}


\twocolumn[
\icmltitle{Appendix: Reviving and Improving Recurrent Back-Propagation}




\begin{icmlauthorlist}
\icmlauthor{Renjie Liao}{A,B,C}
\icmlauthor{Yuwen Xiong}{A,B}
\icmlauthor{Ethan Fetaya}{A,C}
\icmlauthor{Lisa Zhang}{A,C}
\icmlauthor{KiJung Yoon}{D}
\icmlauthor{Xaq Pitkow}{D,E}
\icmlauthor{Raquel Urtasun}{A,B,C}
\icmlauthor{Richard Zemel}{A,C,F}
\end{icmlauthorlist}

\icmlaffiliation{A}{Department of Computer Science, University of Toronto}
\icmlaffiliation{B}{Uber ATG Toronto}
\icmlaffiliation{C}{Vector Institute}
\icmlaffiliation{D}{Department of Electrical and Computer Engineering, Rice University}
\icmlaffiliation{E}{Department of Neuroscience, Baylor College of Medicine}
\icmlaffiliation{F}{Canadian Institute for Advanced Research}
\icmlcorrespondingauthor{Renjie Liao}{rjliao@cs.toronto.edu}





\vskip 0.3in
]




A similar technique to RBP was discovered in physics by Richard Feynman~\cite{feynman1939forces} in modeling molecular forces back in the 1930's. When the energy of molecules are in steady state, the forces on the molecules are defined as the gradient of energy w.r.t. the position parameters of molecules.

\section{Assumptions of RBP}

In this section, we will further discuss the assumptions imposed by RBP.

\subsection{Contraction Map}
Contraction map is often adopted for constructing a convergent dynamic system. 
But it also largely restricts the model capacity and is also hard to satisfy for general neural networks.
Moreover, as pointed out by~\cite{li2015gated}, on a special cycle graph, contraction map will make the impact of one node on the other decay exponentially with their distance.

\subsection{Local Regularization At Convergence}
Recall that in order to apply implicit function theorem, we just to need to make sure that no singular value of the Jacobian is zero. 
In particular, note that $\vert \text{det}(I - J_{F, h^{\ast}}) \vert > 0$ is equivalent to $\vert \text{det}(I - J_{F, h^{\ast}}) \vert^{2} > 0$, one can equivalently rewrite the condition \RNum{2} as,
\begin{align}
\vert \text{det}(I - J_{F, h^{\ast}}) \vert^{2} & = \prod_{i} \vert \sigma_{i}(I - J_{F, h^{\ast}}) \vert^{2} > 0.
\end{align}
Note that for any square matrix $A$, we have,
\begin{align}
\text{det}(A^{\top} A) = \text{det}(A^{\top}) \text{det}(A) = \text{det}(A)^{2}
\end{align}
Therefore, we can instead focus on the positive semi-definite matrix $(I - J_{F, h^{\ast}})^{\top} (I - J_{F, h^{\ast}})$.
The condition can be equivalently stated as below,
\begin{align}
\lambda_{\text{min}}\left( (I - J_{F, h^{\ast}})^{\top} (I - J_{F, h^{\ast}}) \right) > 0,
\end{align}
where $\lambda_{\text{min}}$ is the smallest eigenvalue.
We now briefly discuss two ways of maximizing the smallest eigenvalue.

\paragraph{Maximizing Lower Bound} One way to achieve this is to enforce the lower bound of $\lambda_{\text{min}}$ is larger than zero.
Specifically, according to Gershgorin Circle Theorem, if $A$ is positive definite, we have,
\begin{align}
\lambda_{\text{min}}\left( A \right) \ge 1 - \Vert A - I \Vert_{\infty} \ge 1 - \sqrt{n} \Vert A - I \Vert_{F}.
\end{align}
We can instead maximize the lower bound by adding the term $\max\left(0, \sqrt{n} \Vert A - I \Vert_{F} - 1 \right)$ to the loss function.
One may need to add a small constant to $A$ if $A$ is only positive semi-definite rather than positive definite.
Note that the RHS term is not necessarily larger than zero.

\paragraph{Direct Maximizing By Differentiating Through Lanczos}

Another possible solution is to treat Lanczos algorithm as a fix computational graph to compute the smallest eigenvalue of $(I - J_{F, h^{\ast}})^{\top} (I - J_{F, h^{\ast}})$. 
The most expansive operator in one step Lanczos is the matrix-vector product $(I - J_{F, h^{\ast}})^{\top} (I - J_{F, h^{\ast}})v$ which has doubled complexity as back-propagation.
Differentiating through Lanczos via BPTT is even more expansive which also provides rooms for applying RBP. 
We can add a term $\max\left(0, -\lambda_{\text{min}} \right)$ to the loss function.
Note that the computational complexity of this method is generally high which seems to be impractical for large scale problems.

\section{Recurrent Back-Propagation based on Neumann Series}

In this section, we restate the propositions and prove them.
\begin{prop}\label{prop:neumann_rbp_1}
Assume that we have a convergent RNN which satisfies the implicit function theorem conditions.
If the Neumann series $\sum_{t=0}^{\infty} J_{F, h^{\ast}}^{t}$ converges, then the full Neumann-RBP is equivalent to BPTT. 
\end{prop}

\begin{proof}
Since Neumann series $\sum_{t=0}^{\infty} J_{F, h^{\ast}}^{t}$ converges, we have $(I - J_{F, h^{\ast}})^{-1} = \sum_{t=0}^{\infty} J_{F, h^{\ast}}^{t}$. By substituting it into Eq.~(8), we have,
{\footnotesize
\begin{align}\label{eq:full_rbp}
\frac{\partial L}{\partial w_F} & = \frac{\partial L}{\partial y} \frac{\partial y}{\partial h^{\ast}} \left(I - J_{F, h^{\ast}} \right)^{-1} \frac{\partial F(x, w_F, h^{\ast})}{\partial w_F} \nonumber \\
& = \frac{\partial L}{\partial y} \frac{\partial y}{\partial h^{\ast}} \left(I + J_{F, h^{\ast}} + J_{F, h^{\ast}}^2 + \dots \right) \frac{\partial F(x, w_F, h^{\ast})}{\partial w_F} \nonumber \\
& = \frac{\partial L}{\partial y} \sum_{k=0}^{\infty} \frac{\partial y}{\partial h^{\ast}} J_{F, h^{\ast}}^{k} \frac{\partial F(x, w_F, h^{\ast})}{\partial w_F}.
\end{align}}
Therefore, the full Neumann-RBP is equivalent to original RBP which is further equivalent to BPTT due the implicit function theorem. 
\end{proof}

\begin{prop}\label{prop:neumann_rbp_2}
For the above RNN, let us denote its convergent sequence of hidden states as $h^{0}, h^{1}, \dots, h^{T}$ where $h^{\ast} = h^{T}$ is the steady state.
If we further assume that there exists some step $K$ where $0 < K \le T$ such that $h^{\ast} = h^{T} = h^{T-1} = \dots = h^{T-K}$, then $K$-step Neumann-RBP is equivalent to $K$-step TBPTT.
\end{prop}
\begin{proof}
Since Neumann series $\sum_{t=0}^{\infty} J_{F, h^{\ast}}^{t}$ converges, we have $(I - J_{F, h^{\ast}})^{-1} = \sum_{t=0}^{\infty} J_{F, h^{\ast}}^{t}$. By substituting it into Eq.~(8) and truncate $K$ steps from the end, then the gradient of TBPTT is  
{\footnotesize
\begin{align}\label{eq:k_step_rbp}
\frac{\partial L}{\partial w_F} & = \frac{\partial L}{\partial y} \frac{\partial y}{\partial h^{\ast}} \sum_{k=0}^{K} \left( \prod_{i=T}^{T-k} J_{F, h^{i}} \right) \frac{\partial F(x, w_F, h^{T-k})}{\partial w_F}. \nonumber \\
& = \frac{\partial L}{\partial y} \sum_{k=0}^{K} \frac{\partial y}{\partial h^{\ast}} J_{F, h^{\ast}}^{k} \frac{\partial F(w_F, h^{\ast})}{\partial w_F},
\end{align}}
where the second row uses the fact that $h^{\ast} = h^{T} = h^{T-1} = \dots = h^{T-K}$.
Comparing Eq.~(\ref{eq:full_rbp}) and Eq.~(\ref{eq:k_step_rbp}), it is clear that we exactly recover the $K$-step Neumann-RBP.
\end{proof}

\begin{prop}
If the Neumann series $\sum_{t=0}^{\infty} J_{F, h^{\ast}}^{t}$ converges, then the error between $K$-step and full Neumann series is as following,
{\footnotesize
\begin{align}
\left\Vert \sum_{t=0}^{K} J_{F, h^{\ast}}^{t} - (I - J_{F, h^{\ast}})^{-1} \right\Vert \le \left\Vert (I - J_{F, h^{\ast}})^{-1} \right\Vert \left\Vert J_{F, h^{\ast}} \right\Vert^{K+1} \nonumber 
\end{align}}
\end{prop}
\begin{proof}
First note that, 
{\scriptsize
\begin{align}
(I - J_{F, h^{\ast}}) \left( \sum_{t=0}^{K} J_{F, h^{\ast}}^{t} \right) & = I \left( \sum_{t=0}^{K} J_{F, h^{\ast}}^{t} \right) - J_{F, h^{\ast}} \left( \sum_{t=0}^{K} J_{F, h^{\ast}}^{t} \right) \nonumber \\
& = I - J_{F, h^{\ast}}^{K+1}.
\end{align}}
Multiplying $(I - J_{F, h^{\ast}})^{-1}$ on both sides, we get,
{\scriptsize
\begin{align}
\left( \sum_{t=0}^{K} J_{F, h^{\ast}}^{t} \right) = (I - J_{F, h^{\ast}})^{-1} \left( I - J_{F, h^{\ast}}^{K+1} \right).
\end{align}}
With a bit rearrange, we have, 
{\scriptsize
\begin{align}
\left( \sum_{t=0}^{K} J_{F, h^{\ast}}^{t} \right) - (I - J_{F, h^{\ast}})^{-1} & = - (I - J_{F, h^{\ast}})^{-1} J_{F, h^{\ast}}^{K+1}.
\end{align}}
The result is then straightforward by using Cauchy-Schwarz inequality.
\end{proof}

We further prove the following proposition regarding to the relationship between Neumann-RBP and the original RBP algorithm.
\begin{prop}
$K+1$-step RBP algorithm returned the same gradient with $K$-step Neumann-RBP if $z_0$ in original RBP is initialized as a zero vector. 
\end{prop}
\begin{proof}
To prove this proposition, we only need to compare the vector $z_{K+1}$ and $g_{K}$ returned by two algorithms respectively.
For original RBP, recall in Algorithm $1$, we have the following recursion,
\begin{align}
z_{t} = J_{F, h^{\ast}}^{\top} z_{t-1}  + \left( \frac{\partial L}{\partial y} \frac{\partial y}{\partial h^{\ast}} \right)^{\top}.
\end{align}
Therefore, after $K+1$ step, we have,
\begin{align}
z_{K+1} = \left( J_{F, h^{\ast}}^{\top} \right)^{K+1} z_{0} + \sum_{t=0}^{K} \left( J_{F, h^{\ast}}^{\top} \right)^{t} \left( \frac{\partial L}{\partial y} \frac{\partial y}{\partial h^{\ast}} \right)^{\top}.
\end{align}
For Neumann-RBP, we have the following recursion from Algorithm $2$,
\begin{align}
v_{t} &= J^{\top} v_{t-1} \nonumber \\
g_{t} &= g_{t-1} + v_{t} 
\end{align}
with $v_{0} = g_{0} = \left( {\frac{\partial L}{\partial y}} {\frac{\partial y}{\partial h^{\ast}}} \right)^{\top}$.
Therefore, after $K$ step, we have the following expansion,
\begin{align}
g_{K} = \sum_{t=0}^{K} \left( J_{F, h^{\ast}}^{\top} \right)^{t} \left( \frac{\partial L}{\partial y} \frac{\partial y}{\partial h^{\ast}} \right)^{\top}.
\end{align}
The relationship is now obvious.
\end{proof}

\section{Experiments}

\subsection{Example Code}
Our Neumann-RBP is very simple to implement as long as the auto-differentiation function is provided.
Here we show an example code based on PyTorch in Listing 1.
The effective number of lines is less than $10$.
{
\renewcommand*\familydefault{\ttdefault} 
\begin{lstlisting}[language=Python, float=*b, caption=PyTorch example code]
def neumann_rbp(weight, hidden_state, loss, rbp_step)
  # get the gradient of last hidden state
  grad_h = autograd.grad(loss, hidden_state[-1], retain_graph=True)

  # set v, g to grad_h
  neumann_v = grad_h.clone()
  neumann_g = grad_h.clone()

  for i in range(rbp_step):
    # set last hidden_state's gradient to neumann_v[prev]
    # and get the gradient of last second hidden state
    neumann_v = autograd.grad(
                           hidden_state[-1], hidden_state[-2],
                           grad_outputs=neumann_v,
                           retain_graph=True)
                           
    neumann_g += neumann_v

  # set last hidden_state's gradient to neumann_g
  # and return the gradient of weight
  return autograd.grad(hidden_state[-1], weight, grad_outputs=neumann_g)
\end{lstlisting}
}

\subsection{Continuous Hopfield Network}

\begin{table}[t]
\centering
\resizebox{\linewidth}{!}{%
\begin{tabular}{@{}c|cccc@{}}
\hline
\toprule
Truncate Step & TBPTT & RBP & CG-RBP & Neumann-RBP \\
\midrule
\midrule
10 & 100\% & 1\% & 100\% & 100\% \\
20 & 100\% & 4\% & 100\% & 100\% \\
30 & 100\% & 99\% & 100\% & 100\% \\
\bottomrule
\end{tabular}
}
\caption{Success rate of different methods with different truncation steps. RBP is unstable until the truncation step reaches $30$.}
\label{table:hopfield}
\end{table}

The success rates out of $100$ experiments with different random corruptions and initialization are counted in Table~\ref{table:hopfield}. 
We consider one trial as successful if its final training loss is less than $50\%$ of the initial loss.
From the table, we can see that original RBP almost always fails to converge until the truncation step increases to $30$ whereas both CG-RBP and Neumann-RBP have no issues to converge.

Figure~\ref{fig:hopfield_vis} shows full results of visualization of associative memory.
\begin{figure*}
\centering
\renewcommand*{\arraystretch}{0.1}
\begin{tabular}
	{@{\hspace{0.1mm}}c
	 @{\hspace{0.1mm}}c
	 @{\hspace{0.1mm}}c
	 @{\hspace{0.1mm}}c
	 @{\hspace{0.1mm}}c
	 @{\hspace{0.1mm}}c
	 @{\hspace{0.1mm}}c
	 @{\hspace{0.1mm}}c
	 @{\hspace{0.1mm}}c
	 @{\hspace{0.1mm}}c
	 @{\hspace{0.1mm}}c
	 @{\hspace{0.1mm}}c}
    \includegraphics[width=0.08\linewidth]{imgs/hopfield/image_corrupt_0.png}&
    \includegraphics[width=0.08\linewidth]{imgs/hopfield/bptt_recover_0.png}&
    \includegraphics[width=0.08\linewidth]{imgs/hopfield/tbptt_recover_0.png}&
    \includegraphics[width=0.08\linewidth]{imgs/hopfield/rbp_recover_0.png}&
    \includegraphics[width=0.08\linewidth]{imgs/hopfield/cg_rbp_recover_0.png} &
    \includegraphics[width=0.08\linewidth]{imgs/hopfield/neumann_rbp_recover_0.png}  &
    \includegraphics[width=0.08\linewidth]{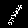}&
    \includegraphics[width=0.08\linewidth]{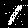}&
    \includegraphics[width=0.08\linewidth]{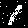}&
    \includegraphics[width=0.08\linewidth]{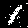}&
    \includegraphics[width=0.08\linewidth]{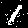} &
    \includegraphics[width=0.08\linewidth]{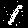}\\
    \includegraphics[width=0.08\linewidth]{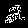}&
    \includegraphics[width=0.08\linewidth]{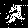}&
    \includegraphics[width=0.08\linewidth]{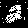}&
    \includegraphics[width=0.08\linewidth]{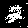}&
    \includegraphics[width=0.08\linewidth]{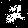} &
    \includegraphics[width=0.08\linewidth]{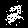} &
    \includegraphics[width=0.08\linewidth]{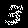}&
    \includegraphics[width=0.08\linewidth]{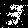}&
    \includegraphics[width=0.08\linewidth]{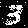}&
    \includegraphics[width=0.08\linewidth]{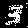}&
    \includegraphics[width=0.08\linewidth]{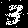} &
    \includegraphics[width=0.08\linewidth]{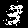}\\
    \includegraphics[width=0.08\linewidth]{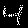}&
    \includegraphics[width=0.08\linewidth]{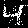}&
    \includegraphics[width=0.08\linewidth]{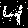}&
    \includegraphics[width=0.08\linewidth]{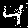}&
    \includegraphics[width=0.08\linewidth]{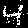} &
    \includegraphics[width=0.08\linewidth]{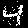} &
    \includegraphics[width=0.08\linewidth]{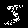}&
    \includegraphics[width=0.08\linewidth]{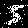}&
    \includegraphics[width=0.08\linewidth]{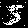}&
    \includegraphics[width=0.08\linewidth]{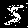}&
    \includegraphics[width=0.08\linewidth]{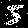} &
    \includegraphics[width=0.08\linewidth]{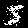}\\
    \includegraphics[width=0.08\linewidth]{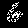}&
    \includegraphics[width=0.08\linewidth]{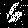}&
    \includegraphics[width=0.08\linewidth]{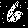}&
    \includegraphics[width=0.08\linewidth]{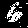}&
    \includegraphics[width=0.08\linewidth]{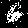} &
    \includegraphics[width=0.08\linewidth]{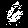} &                 
    \includegraphics[width=0.08\linewidth]{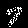}&
    \includegraphics[width=0.08\linewidth]{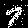}&
    \includegraphics[width=0.08\linewidth]{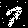}&
    \includegraphics[width=0.08\linewidth]{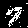}&
    \includegraphics[width=0.08\linewidth]{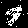} &
    \includegraphics[width=0.08\linewidth]{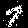}\\
    \includegraphics[width=0.08\linewidth]{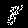}&
    \includegraphics[width=0.08\linewidth]{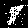}&
    \includegraphics[width=0.08\linewidth]{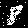}&
    \includegraphics[width=0.08\linewidth]{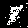}&
    \includegraphics[width=0.08\linewidth]{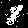} &
    \includegraphics[width=0.08\linewidth]{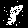} &
    \includegraphics[width=0.08\linewidth]{imgs/hopfield/image_corrupt_9.png}&
    \includegraphics[width=0.08\linewidth]{imgs/hopfield/bptt_recover_9.png}&
    \includegraphics[width=0.08\linewidth]{imgs/hopfield/tbptt_recover_9.png}&
    \includegraphics[width=0.08\linewidth]{imgs/hopfield/rbp_recover_9.png}&
    \includegraphics[width=0.08\linewidth]{imgs/hopfield/cg_rbp_recover_9.png} &
    \includegraphics[width=0.08\linewidth]{imgs/hopfield/neumann_rbp_recover_9.png}\\
    (a) & (b) & (c) & (d) & (e) & (f) & (a) & (b) & (c) & (d) & (e) & (f) \\
\end{tabular}
\caption{Visualization of associative memory. (a) Corrupted input image; (b)-(f) are retrieved images by BPTT, TBPTT, RBP, CG-RBP, Neumann-RBP respectively.} 
\label{fig:hopfield_vis}
\end{figure*}

\subsection{Citation Networks}

Table~\ref{exp:datasets} shows the statistics of datasets we used in our experiments.

\begin{table}[!htbp]
\centering
\resizebox{\linewidth}{!}{%
\begin{tabular}{@{}l|rrrrr@{}}
\hline
\toprule
Dataset & \#Nodes & \#Edges & \#Classes & \#Features \\
\midrule
Cora & 2,708 & 5,429 & 7 & 1,433 \\
Pubmed & 19,717 & 44,338 & 3 & 500 \\
\bottomrule
\end{tabular}
}
\caption{Dataset statistics of citation networks.}
\label{exp:datasets}
\end{table}

We also inlcude the comparsion with the recent work ARTBP~\cite{tallec2017unbiasing}. 
The experiment setting is exactly the same as described in the paper. 
Since the underlying RNN has the loss defined at the last time step, i.e., 100th step, we adapt the ARTBP as follows: instead of randomly truncating at multiple locations, we randomly choose one time step to truncate. 
Similar analysis can be derived to compensate the truncated gradient such that it is unbiased. 
Due to the limited time, we only tried uniform and truncated Poisson distribution (expected truncation point is roughly at the 95th time step which is where TBPTT stops) over the truncation location. 
We use SGD with momentum as the optimizer for all methods. 
The average validation accuracy over 10 runs are in the table below. 
We can see that both ARTBP variants do not perform as well as Neumann-RBP in this setting. 
ARTBP with truncated Poisson is better than the one with uniform which matches the other observation that TBPTT is better than full BPTT. 

\begin{table}[!htbp]
\centering
\begin{tabular}{@{}c|cc@{}}
\hline
\toprule
Test Acc. & Cora \\
\midrule
\midrule
Baseline & 39.96 $\pm$ 3.4 \\
BPTT & 24.48 $\pm$ 6.6 \\
TBPTT & 46.55 $\pm$ 6.4 \\
Uniform-ARTBP & 27.88 $\pm$ 3.2 \\
TPoisson-ARTBP & 42.22 $\pm$ 7.1 \\
RBP & 29.25 $\pm$ 3.3 \\
CG-RBP & 39.26 $\pm$ 6.5 \\
Neumann-RBP & \textbf{46.63} $\pm$ 8.3 \\
\bottomrule
\end{tabular}
\caption{Test accuracy of different methods on citation networks.} 
\label{table:citation}
\end{table}

\bibliography{rbp}
\bibliographystyle{icml2018}